\documentclass{article}

\usepackage[preprint]{neurips_2026}

\usepackage[utf8]{inputenc}
\usepackage[T1]{fontenc}
\usepackage{hyperref}
\usepackage{url}
\usepackage{booktabs}
\usepackage{amsfonts}
\usepackage{amsmath}
\usepackage{amssymb}
\usepackage{nicefrac}
\usepackage{microtype}
\usepackage{xcolor}
\usepackage{graphicx}
\usepackage{multirow}
\usepackage{subcaption}
\usepackage{enumitem}
\usepackage{amsthm}
\newtheorem{definition}{Definition}
\newtheorem{proposition}{Proposition}

\graphicspath{{figures/}}

\newcommand{\pibase}{\pi_{\text{base}}}
\newcommand{\pisft}{\pi_{\text{SFT}}}
\newcommand{\pirl}{\pi_{\text{RL}}}
\newcommand{\piexplore}{\pi_{\text{RL+explore}}}

\title{Does RL Expand the Capability Boundary of LLM Agents? A \textsc{Pass}@$(k,T)$ Analysis\thanks{Code available at \url{https://github.com/zhiyuanZhai20/pass-kt-analysis}.}}

\author{
  Zhiyuan Zhai \\
  Fudan University \\
  \texttt{22110720067@m.fudan.edu.cn} \\
  \And
  Wenjing Yan \\
  Chinese University of Hong Kong \\
  \texttt{wenjingyan@cuhk.edu.hk} \\
  \And
  Xiaodan Shao \\
  University of Waterloo \\
  \texttt{x6shao@uwaterloo.ca} \\
  \And
  Xin Wang \\
  Fudan University \\
  \texttt{xwang11@fudan.edu.cn}
}

\begin{document}

\maketitle

\begin{abstract}
Does reinforcement learning genuinely expand what LLM agents can do, or does it merely make them more reliable at what they could already do? For \emph{static} mathematical reasoning, a recent line of work answers the second: base and RL pass@$k$ curves converge at large $k$, and RL only redistributes probability mass within the base model's existing capability set. We ask whether this pessimistic reading survives the move to \emph{agentic} tool use, where an agent's $T$ rounds of environment interaction enable compositional strategies that re-sampling cannot recover. We introduce \textsc{Pass}@$(k,T)$, a two-dimensional evaluation metric that jointly varies a sampling budget $k$ and an interaction-depth budget $T$, cleanly separating capability expansion from efficiency improvement. Our main finding is that, contrary to the static-reasoning result, tool-use RL genuinely enlarges the agent's capability boundary: as the sampling budget grows, the RL agent's pass-curve pulls above the base model's and the gap \emph{widens} at the right tail rather than converging, the opposite of what the static-reasoning RLVR literature reports. The expansion is specific to tasks that require compositional, sequential information gathering; on simpler categories RL behaves as prior work predicts. Under matched training data, supervised fine-tuning produces the opposite effect: it \emph{regresses} the capability boundary on the same compositional tasks, which isolates self-directed exploration, not data exposure, as the causal factor. A mechanism analysis explains how the expansion arises: RL reweights the base model's existing strategy distribution toward the subset whose downstream reasoning more often yields a correct answer, with the improvement concentrated on how the agent \emph{integrates} retrieved information rather than on what it searches for. These results reconcile the optimistic and pessimistic readings of RL for LLMs: both are correct, on different task types. In agentic settings, when the environment admits compositional solutions the base model cannot stitch together alone, RL meaningfully teaches new capabilities.
\end{abstract}


\section{Introduction}

LLM-based agents that interact with external environments through tool use represent one of the
most rapidly advancing frontiers in AI \citep{yao2023react,schick2023toolformer,qin2024toolllm,nakano2022webgpt}. Built on large-scale pre-trained models \citep{brown2020gpt3,openai2023gpt4,touvron2023llama}, agents operate in a loop (observe, act, receive feedback, and act again), enabling tasks that require external knowledge or multi-step information gathering \citep{mialon2023augmented,wang2024survey}.

Reinforcement learning has emerged as the method of choice for training such agents. Building on RLHF for instruction following \citep{ouyang2022instructgpt,christiano2017deep} and RL-based reasoning improvements \citep{guo2025deepseekr1,openai2024reasoning}, recent systems
including Agent-R1 \citep{zhang2025agentr1}, ReTool \citep{feng2025retool}, and Agent-Q \citep{putta2024agentq} apply RL end-to-end, achieving impressive results on benchmarks spanning web navigation, code generation, and tool-augmented reasoning.

But beneath these empirical successes lies a question that has received surprisingly little scrutiny:

\begin{center}
\textit{Does RL expand what LLM agents can do, or merely make them more reliable at what they can already do?}
\end{center}

When an RL-trained agent outperforms its base model on a benchmark, the improvement could
stem from two fundamentally different sources. The first is \textbf{capability expansion}: RL teaches the
agent to solve problems that the base model cannot solve at all, for instance, by discovering novel
tool-use strategies that access information the base model would never seek out. The second is
\textbf{efficiency improvement}: RL makes the agent more reliably solve problems that the base model can
already solve, by increasing the probability of selecting effective strategies. These two sources of
improvement have entirely different implications for the field. If RL primarily expands capabilities,
then investing in better RL algorithms and exploration mechanisms is the path forward. If RL
primarily improves efficiency, then the capability ceiling is determined by the base model, and effort
is better spent on pre-training and data quality.

Distinguishing between these two effects requires evaluation tools that go beyond standard accuracy
metrics. A model's accuracy (pass@1) conflates capability and efficiency: an improvement from 30\%
to 60\% accuracy could mean the model learned to solve new problems, or it could mean the model
became more consistent at solving problems it could already occasionally solve. Existing agentic
benchmarks \citep{liu2024agentbench,jimenez2024swebench,dua2019drop,zheng2024gptfathom} report aggregate accuracy numbers and do not
attempt this decomposition.

In this paper, we introduce a framework that enables this decomposition. Our key insight is that
agent capability is inherently two-dimensional: it depends on both (1) how many independent
attempts the agent is given, and (2) how deeply the agent can interact with the environment in each
attempt. We formalize this through PASS@(k,T), a metric that evaluates whether an agent can solve a
problem given $k$ independent trajectories, each with at most $T$ rounds of environment interaction.
This two-dimensional structure naturally separates efficiency from capability: varying $k$ reveals
sampling efficiency (whether more attempts help), while varying $T$ reveals interaction capability
(whether deeper environment engagement unlocks new solutions). The interplay between $k$ and $T$
further reveals whether the value of RL lies in better sampling or better interaction, a question that
no existing metric can address.

We make three contributions.

\begin{itemize}
\item \textbf{The first two-dimensional agentic evaluation framework.} We introduce \textsc{Pass}@$(k, T)$, the first metric that evaluates an LLM agent on a sampling budget $k$ and an interaction-depth budget $T$ jointly. The $T$ axis has no analog in static pass@$k$ analysis and is what makes the agentic capability question answerable: it disentangles, for the first time, whether an improvement from training comes from sampling reliability or from productive use of deeper environment interaction. We establish the metric's formal properties (monotonicity, reduction to static pass@$k$, a strict-expansion corollary) and show that it is \emph{necessary}, not just convenient, to distinguish capability expansion from efficiency improvement in the agentic regime.

\item \textbf{The first empirical demonstration that agentic RL expands the capability boundary.} Using \textsc{Pass}@$(k,T)$ with \emph{matched} training data across base, SFT, and GRPO-trained agents, we establish that tool-use RL \emph{does} expand an LLM agent's capability boundary on compositional tool-use tasks. The signature is not an improvement in pass@1 (at the smallest sampling budgets, $\pibase$ and $\pirl$ are essentially tied) but a \emph{widening} gap as the sampling budget grows: the RL agent's pass-curve pulls above $\pibase$'s as $k$ increases and the gap reaches its maximum at the capability-boundary limit, the \emph{opposite} of the convergence pattern reported for static-reasoning RLVR. This directly contradicts the pessimistic conclusion that the static-reasoning literature has been taken to imply for LLM agents in general. Matched-data SFT produces the opposite effect on the same tasks, which identifies the learning signal, not data exposure, as the causal factor.

\item \textbf{A mechanistic account of how the expansion arises.} We give a three-pronged mechanism analysis (perplexity decomposition, cross-policy swapping, and strategy-diversity measurement) and use it to characterize the training intervention: RL reweights the base model's existing strategy distribution toward the subset whose downstream reasoning more often produces a correct answer, with the improvement concentrated on how the agent \emph{integrates} retrieved information rather than on what it searches for. This is a sharper account than the standard ``RL optimizes for reward'' story and has direct implications for where reward design and exploration bonuses can most usefully be applied.
\end{itemize}

We discuss related work in Appendix~\ref{app:related}, covering RL for LLM agents, tool-augmented benchmarks, and the pass@$k$-based capability-evaluation debate.

\section{The PASS@(k, T) Framework}
\label{sec:framework}

\subsection{Motivation: why agent evaluation needs two dimensions}

Accuracy (pass@1) conflates two different questions: (Q1) whether the agent can ever solve the problem, and (Q2) how reliably it does so among solvable problems. Pass@$k$ resolves this ambiguity in the \emph{static} setting by varying the sampling budget $k$: at large $k$, pass@$k$ converges to the capability indicator, and \citet{yue2025rlvr,wang2025limits} use this observation to argue that RLVR on mathematical reasoning redistributes mass within the base model's capability set but does not enlarge it. Related approaches based on repeated sampling \citep{wang2023selfconsistency,snell2024scaling,cobbe2021gsm8k} and process reward models \citep{lightman2024lets} similarly focus on the $k$ axis alone. In the \emph{agentic} setting, pass@$k$ is insufficient because it treats each attempt as a black box with a fixed interaction budget, and cannot answer a third question unique to tool-using agents: (Q3) \emph{does deeper environment interaction unlock problems that cannot be solved at shallow interaction, no matter how many times one re-samples?} An agent with $T$ rounds of tool use can execute compositional strategies, such as chaining one retrieval's output into the next retrieval's input or verifying intermediate results \citep{trivedi2022musique,khot2023decomposed}, that are structurally unavailable at shallower $T$; a bridge question such as ``What is the nationality of the director of Film~$X$?'' cannot be answered with a single retrieval because the second query's argument is a function of the first query's result, and no amount of re-sampling at $T = 1$ can compensate. We therefore parametrize capability by both the sampling budget $k$ and the interaction-depth budget $T$: the $k$-axis probes sampling reliability (Q2) and the $T$-axis probes interaction capability (Q3).

\subsection{Definition}

We formalize the agent-environment interaction as follows. A policy $\pi$ interacts with an environment for up to $T$ rounds. At round $t$, the agent observes a state $s_t$ (containing the task, the action history, and all previous observations), selects an action $a_t$ (either a tool call or a final answer), and receives an observation $o_t = \mathsf{env}(s_t, a_t)$. A trajectory $\tau = (s_0, a_0, o_0, s_1, a_1, o_1, \ldots, s_T)$ records the full interaction. The trajectory succeeds iff its final action is a correct answer, i.e., $\mathrm{EM}(\mathsf{answer}(\tau), y^\star(q)) = 1$ for gold answer $y^\star$. For a fixed interaction budget $T$, let $c_T(q, \pi; n)$ be the number of successful trajectories among $n$ i.i.d.\ rollouts of $\pi$ on problem $q$.

\textsc{Pass}@$(k, T)$ answers the following question: if we run the agent $k$ independent times on $q$, giving each run at most $T$ rounds of environment interaction, what is the probability that \emph{at least one} of the $k$ runs produces a correct final answer? It is therefore the capability of policy $\pi$ on $q$ under a \emph{compound} budget ($k$ attempts on one axis, $T$ interaction rounds on the other) and tells us what the agent can achieve when both its sampling and its interaction resources are bounded.

\begin{definition}[PASS@$(k, T)$]
\label{def:passkt}
For a problem $q$, policy $\pi$, interaction depth $T \geq 0$, and sampling budget $k \leq n$,
\begin{equation}
\textsc{Pass}@(k, T)(q, \pi) \;\triangleq\; \Pr_{\tau_1, \ldots, \tau_k \sim \pi(\cdot \mid q; T)}\!\left[\exists\, i \in [k] : \tau_i \text{ is correct}\right] \;=\; 1 - \frac{\binom{n - c_T}{k}}{\binom{n}{k}}.
\end{equation}
The right-hand form is the unbiased hypergeometric estimator \citep{chen2021codex} computed from $n$ i.i.d.\ rollouts of $\pi$ at depth $T$ with $c_T$ successes, and is what we use throughout. When $k > n - c_T$, the binomial coefficient is taken as $0$ and $\textsc{Pass}@(k, T) = 1$.
\end{definition}

The two axes isolate two conceptually distinct things an agent can do. The $k$ axis asks ``how reliably does $\pi$ solve $q$?'' At small $k$, \textsc{Pass}@$(k, T)$ tracks pass@1 accuracy and is sensitive to sampling noise; at large $k$, it tells us whether $\pi$ can ever produce a correct answer at all, i.e., whether $q$ lies inside $\pi$'s capability set. The $T$ axis asks ``how much of $q$ requires the environment?'' At $T = 0$ the agent must answer from parametric knowledge alone, and increases in \textsc{Pass}@$(k, T)$ as $T$ grows measure value that cannot be recovered by re-sampling. Two boundary cases sharpen this. When $k \to n$, $\textsc{Pass}@(k, T)(q, \pi) \to \mathbf{1}[c_T > 0]$, the indicator of whether $\pi$ is \emph{capable} of solving $q$ at depth $T$; when $T = 0$ the metric reduces to the standard static pass@$k$ of \citet{chen2021codex}. $\textsc{Pass}@(k, T)$ therefore strictly generalizes pass@$k$, and the two extreme slices of the $(k, T)$ grid recover, respectively, the capability-boundary indicator and the static-reasoning metric.

\subsection{Capability boundary, expansion, and efficiency}
\label{sec:boundary}

The boundary-case $k \to n$ motivates the following structural definitions, which formalize what it means for RL to ``teach'' a policy something the base model does not already know.

\begin{definition}[Capability boundary]
The capability boundary of a policy $\pi$ at interaction depth $T$ is
$\mathcal{B}_T(\pi) \,=\, \{q \in \mathcal{Q} : \lim_{k \to \infty} \textsc{Pass}@(k, T)(q, \pi) > 0\}$,
the set of problems $\pi$ can solve at least once under unlimited re-sampling at depth $T$.
\end{definition}

\begin{definition}[Capability expansion]
$\pirl$ achieves capability expansion over $\pibase$ at depth $T$ iff $\mathcal{B}_T(\pirl) \setminus \mathcal{B}_T(\pibase) \neq \emptyset$: there exists at least one problem that the RL agent solves but the base agent cannot solve at any sampling budget.
\end{definition}

\begin{definition}[Efficiency improvement]
$\pirl$ achieves efficiency improvement on $q \in \mathcal{B}_T(\pibase) \cap \mathcal{B}_T(\pirl)$ at depth $T$ iff $\textsc{Pass}@(1, T)(q, \pirl) > \textsc{Pass}@(1, T)(q, \pibase)$: both models can solve $q$, but the RL model does so more reliably.
\end{definition}

The two definitions are genuinely distinct: efficiency improvement moves probability mass within the intersection of capability sets, while capability expansion enlarges the union. A given training procedure may produce both, either, or neither; disentangling the two is the central measurement problem that this paper addresses.

\subsection{Diagnostic properties of \textsc{Pass}@$(k, T)$}
\label{sec:diagnostic}

The two-dimensional structure of \textsc{Pass}@$(k, T)$ enables diagnostics that no one-dimensional metric can provide.

\paragraph{Monotonicity.} \textsc{Pass}@$(k, T)$ is non-decreasing in both $k$ and $T$: for $k_1 \leq k_2$ and $T_1 \leq T_2$, $\textsc{Pass}@(k_1, T_1) \leq \textsc{Pass}@(k_2, T_2)$. Any trajectory that succeeds at depth $T_1$ remains successful at $T_2 > T_1$ (unused rounds are simply not exercised), and monotonicity in $k$ is the standard hypergeometric property; a full proof is in Appendix~\ref{app:properties}.

\paragraph{Marginal value.} The partial improvements $\Delta_k = \textsc{Pass}@(2k, T) - \textsc{Pass}@(k, T)$ and $\Delta_T = \textsc{Pass}@(k, T+1) - \textsc{Pass}@(k, T)$ quantify the marginal value of doubling the sampling budget versus deepening the interaction budget by one round. When $\Delta_T \gg \Delta_k$, environment interaction provides value that independent re-sampling cannot replicate, and a practitioner on a fixed compute budget should prefer to deepen $T$; when $\Delta_T \ll \Delta_k$, re-sampling dominates.

\paragraph{Interaction saturation.} We define $T^\star(\pi; \epsilon) = \min\{T : \textsc{Pass}@(k_{\max}, T+1) - \textsc{Pass}@(k_{\max}, T) < \epsilon\}$ as the smallest interaction depth beyond which an additional round of interaction buys less than $\epsilon$ at the high-$k$ limit. A higher $T^\star$ for one model than another indicates that the former extracts productive value from deeper retrieval chains; equal $T^\star$ with unequal plateau heights indicates that training has raised the ceiling at a fixed interaction horizon rather than extended the horizon itself.

\paragraph{Reduction to static pass@$k$.} When $T = 0$ the agent has no tool access and $\textsc{Pass}@(k, 0) = \text{pass}@k$, so Category A (MATH-500 with no tool) functions as a sanity check against published static pass@$k$ numbers for the base model and, by construction, measures whether tool-use RL has any collateral effect on parametric reasoning.


\section{Experimental Setup}
\label{sec:setup}

\paragraph{Tasks and retrieval environment.}
We use three task categories of increasing interaction complexity, each with 100 test problems: \textbf{Category A} is MATH-500 \citep{hendrycks2021math} with no tool access, a negative control excluded from training; \textbf{Category B} is HotPotQA comparison questions \citep{yang2018hotpotqa} (two \emph{independent} retrievals from the question itself), and \textbf{Category C} is HotPotQA bridge questions (two \emph{compositionally dependent} retrievals, where the second query's argument must be extracted from the first query's result). Categories B and C each provide 100 training problems from HotPotQA's training split, for a total training set of 200 shared between $\pisft$ and $\pirl$. The agent has a single tool, \texttt{search(query)}, which performs deterministic BM25 \citep{robertson2009bm25} retrieval over the standard HotPotQA distractor corpus (10 paragraphs per question, 2 gold and 8 distractors); identical queries always return the same paragraph. The agent follows a ReAct loop \citep{yao2023react} of alternating \texttt{Thought}/\texttt{Search}/\texttt{Answer} actions with retrieved paragraphs appended as \texttt{Observation}; a trajectory is successful iff its final \texttt{Answer} matches the gold answer under light normalization (lowercase, strip articles and trailing punctuation).

\paragraph{Models under a matched-data comparison.}
All three trained models start from the same base: Qwen2.5-7B-Instruct \citep{qwen25} with search-tool instructions in the system prompt. $\pisft$ is $\pibase$ LoRA-fine-tuned \citep{hu2022lora} on 200 expert trajectories constructed from HotPotQA's gold supporting-fact annotations, with observation tokens masked from the loss. $\pirl$ is $\pibase$ trained with GRPO \citep{shao2024deepseekmath} on the \emph{same 200 problems} with binary exact-match reward, group size $G = 8$, per-token $k_3$ KL estimator \citep{schulman2020kl} against the LoRA-disabled reference, $T_{\mathrm{train}} = 5$, temperature $0.7$, and $10$ epochs (1000 steps). We additionally train an exploration-bonus variant $\piexplore$ with reward $R_{\mathrm{aug}} = R + \lambda\cdot R_{\mathrm{explore}}$, $\lambda = 0.1$, where $R_{\mathrm{explore}} = 1$ if the set of retrieved-paragraph titles in the trajectory is unseen in the current training batch (Appendix~\ref{app:exploration}). The critical design choice is that $\pisft$ and $\pirl$ see the \emph{same} 200 problems and differ only in the learning signal; any $\mathcal{B}_T(\pirl) \setminus \mathcal{B}_T(\pisft)$ asymmetry therefore cannot be attributed to data exposure.

\paragraph{Evaluation.}
For each model, problem, and $T \in \{0, 1, 2, 3, 5\}$ we draw $n = 64$ i.i.d.\ trajectories at temperature $0.7$ using vLLM \citep{kwon2023vllm} and compute $\textsc{Pass}@(k, T)$ on the grid $k \in \{1, 2, 4, 8, 16, 32, 64\}$ via Def.~\ref{def:passkt}. Each $(q, T)$ cell is a separate rollout, so the $T$ sweep is not a restriction of a single long run. Full hyperparameters, expert-trajectory construction, training curves, and per-category PASS@$(k,T)$ tables are in Appendix~\ref{app:training}--\ref{app:fulltable}.

\section{Results}
\label{sec:results}

\subsection{Does RL Expand Agent Capability?}

We begin with the central question. Table~\ref{tab:capability} reports the capability boundary of each model at full evaluation budget ($k = 64$, $T = 5$), measured as the set of problems each model can solve at least once. The critical column is $|\mathcal{B}_{\text{RL}} \setminus \mathcal{B}_{\text{base}}|$: problems that the RL agent can solve but the base agent cannot solve in any of 64 attempts with 5 search rounds.

\begin{table}[t]
\centering
\caption{Capability boundary analysis. Number of problems (out of 100 per category) solvable at PASS@(64, 5) and the overlap between model capability sets. $|\mathcal{B}_X \setminus \mathcal{B}_Y|$ denotes problems solvable by $X$ but not $Y$. Bridge questions (Category C) show the largest RL expansion over the base agent.}
\label{tab:capability}
\begin{tabular}{lcccccc}
\toprule
Category & $|\mathcal{B}_{\text{base}}|$ & $|\mathcal{B}_{\text{SFT}}|$ & $|\mathcal{B}_{\text{RL}}|$ & $|\mathcal{B}_{\text{RL}} \setminus \mathcal{B}_{\text{base}}|$ & $|\mathcal{B}_{\text{base}} \setminus \mathcal{B}_{\text{RL}}|$ & $|\mathcal{B}_{\text{SFT}} \setminus \mathcal{B}_{\text{base}}|$ \\
\midrule
A (Pure Reasoning) & 84 & 89 & 84 & 3 & 3 & 5 \\
B (Independent Retr.) & 82 & 85 & 86 & 5 & 1 & 5 \\
C (Sequential Retr.) & 77 & 73 & \textbf{81} & \textbf{5} & 1 & $-4$ \\
\bottomrule
\end{tabular}
\end{table}
\vspace{-4pt}

Table~\ref{tab:capability} reveals a three-layer pattern that cleanly maps onto the three task categories:

\textbf{Category A (no effect).}\quad On pure mathematical reasoning, $|\mathcal{B}_{\text{RL}}| = |\mathcal{B}_{\text{base}}| = 84$, and $|\mathcal{B}_{\text{RL}} \setminus \mathcal{B}_{\text{base}}| = |\mathcal{B}_{\text{base}} \setminus \mathcal{B}_{\text{RL}}| = 3$. The RL-trained agent neither gains nor loses MATH-500 capability relative to the base model, confirming that tool-use RL does \emph{not} transfer to pure reasoning when mathematics is excluded from the training distribution.

\textbf{Category B (small expansion).}\quad On comparison questions, RL expands the capability boundary modestly: 5 problems become newly solvable while only 1 is lost, yielding a net gain of 4 problems over the base agent. SFT shows a similar net gain of 3 (+5 new, $-$2 lost). RL still provides a clear capability-boundary gain over $\pibase$, but its advantage over $\pisft$ is not significant: comparison questions reduce to two independent retrievals that are straightforwardly demonstrated in the expert trajectories, so supervised imitation is already a nearly sufficient learning signal and leaves little headroom for the exploration advantage of RL to show up.

\textbf{Category C (substantial expansion + SFT regression).}\quad On bridge questions, which require sequential information gathering, the picture is qualitatively different. RL expands the capability boundary by 5 problems with only 1 regression (net +4), reaching $|\mathcal{B}_{\text{RL}}| = 81$, the largest of any model we evaluate. In stark contrast, SFT \emph{contracts} the boundary: $|\mathcal{B}_{\text{SFT}}| = 73$, losing 7 problems the base agent could solve while gaining only 3 (net $-$4). The asymmetry between RL and SFT on Category C is the single sharpest quantitative finding of this study: $|\mathcal{B}_{\text{RL}} \setminus \mathcal{B}_{\text{SFT}}| = 9$ versus $|\mathcal{B}_{\text{SFT}} \setminus \mathcal{B}_{\text{RL}}| = 1$. Given that SFT and RL share identical 200-problem training data, this 9:1 gap can be attributed to the \emph{learning signal} itself (expert imitation vs.\ self-directed exploration under binary reward), not to data exposure.

Aggregated across all 300 problems, the RL agent's total capability boundary is $|\mathcal{B}_{\text{RL}}| = 251$ vs.\ $|\mathcal{B}_{\text{base}}| = 243$ and $|\mathcal{B}_{\text{SFT}}| = 247$. At PASS@(64,5) on Category C, the RL agent reaches 0.81 compared to 0.77 for the base agent and 0.73 for SFT, a +4 percentage-point expansion at the capability-boundary limit on the category where the two-dimensional evaluation matters most.

\subsection{The PASS@$(k,T)$ surface and pass-curves}
\label{sec:landscape}

Figure~\ref{fig:heatmaps} shows the full $(k, T)$ surface as a $3\times 3$ grid of heatmaps (rows = categories, columns = models), and Figure~\ref{fig:passk} collapses the $T = T_{\max}$ slice into pass-curves on a log-$k$ axis. Three qualitative patterns match the three-layer Table~\ref{tab:capability} answer.

\begin{figure}[t]
\centering
\includegraphics[width=0.75\textwidth]{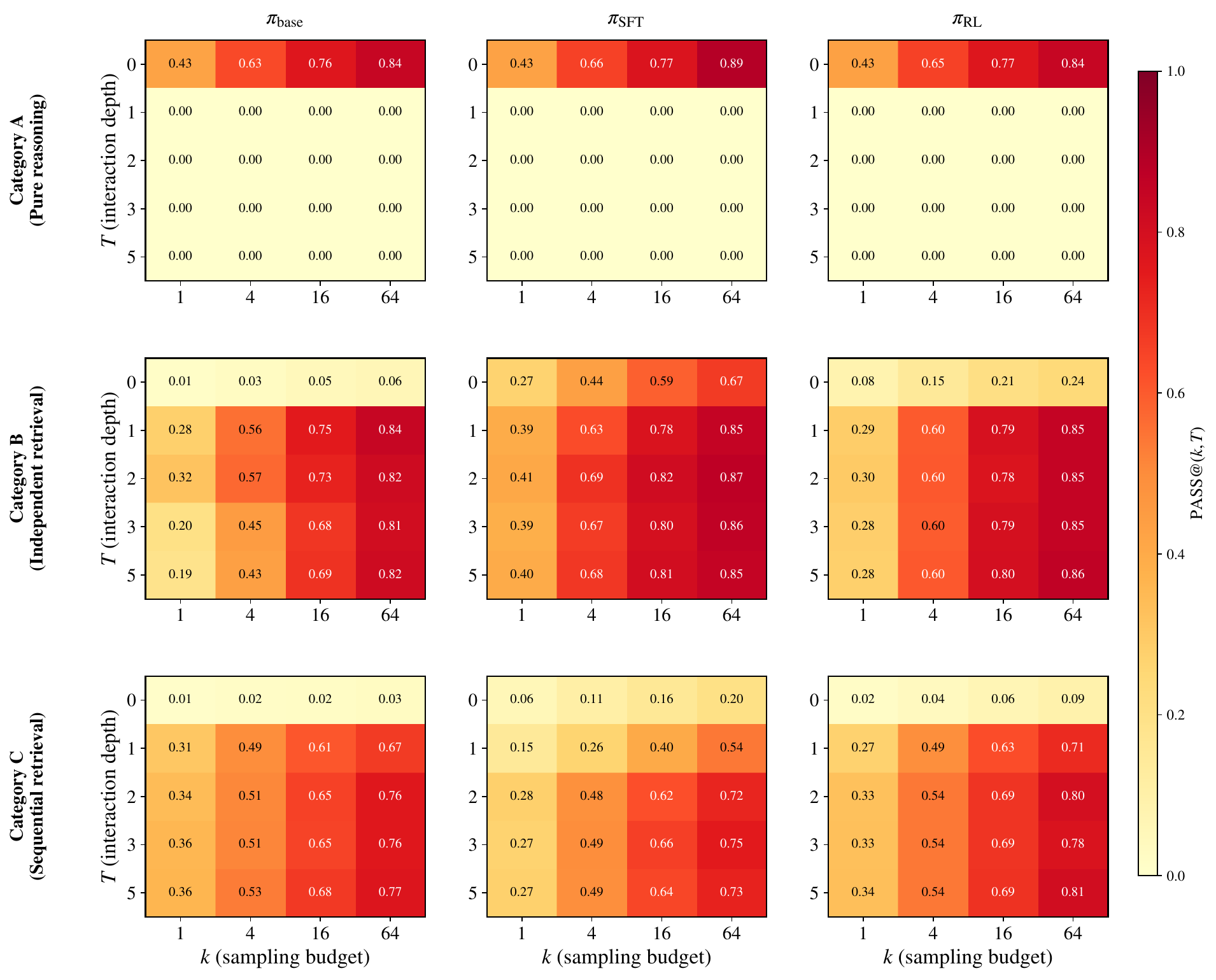}
\caption{The full \textsc{Pass}@$(k, T)$ landscape: rows are task categories, columns are models. On Category~A every row is flat in $T$ (tool unavailable). On Category~B the surface saturates at $T = 2$. On Category~C $\pirl$'s panel is uniformly warmer than $\pibase$'s at $T \geq 2$, while $\pisft$'s is cooler; the training signal shifts the entire surface.}
\label{fig:heatmaps}
\end{figure}

\begin{figure}[t]
\centering
\includegraphics[width=0.92\textwidth]{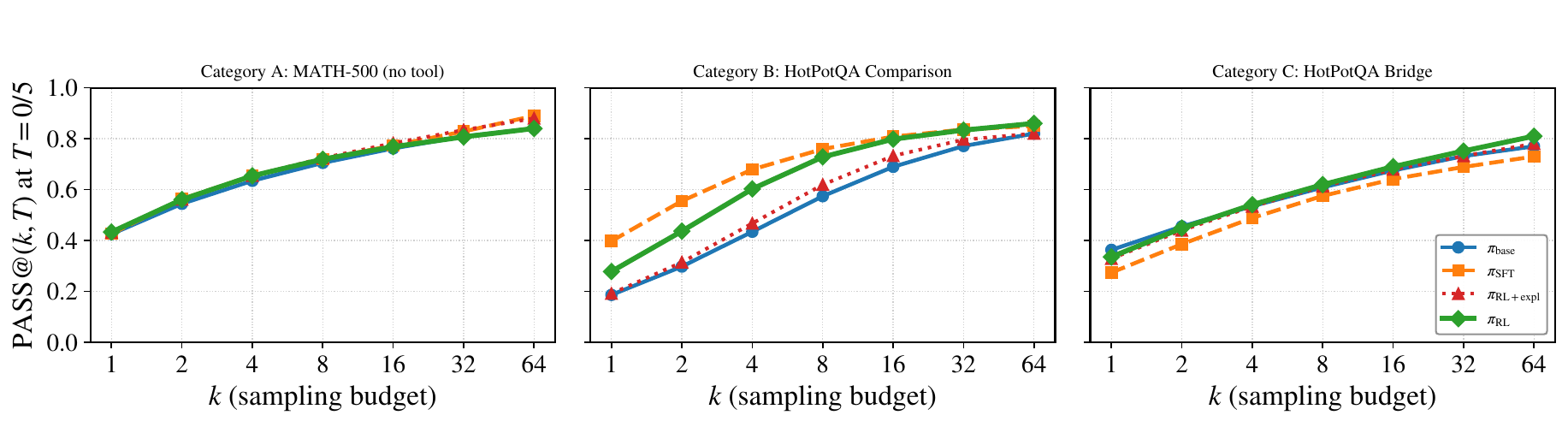}
\caption{$\textsc{Pass}@(k, T_{\max})$ vs.\ sampling budget $k$ on a log axis. $T_{\max} = 0$ for Category A and $T_{\max} = 5$ for B and C. On Category C, $\pibase$ and $\pirl$ cross near $k = 4$: at $k = 1$ $\pibase$ is slightly ahead, but as $k$ grows $\pirl$ pulls above and the gap \emph{widens} ($+4$ pp at $k = 64$), the opposite of the convergence reported by \citet{yue2025rlvr}. $\pisft$ sits below both.}
\label{fig:passk}
\end{figure}

\textbf{(i) Category A is flat in $T$; pass-curves cluster within $\leq 5$ pp with no systematic ordering.} Tool-use RL is orthogonal to parametric MATH-500 reasoning, the expected negative control.

\textbf{(ii) Category B saturates at $T = 2$.} Every model jumps from near-zero at $T = 0$ to its plateau at $T \in \{1, 2\}$ and stays flat. At $k = 64$, $\pirl$ is highest ($0.86$), $\pisft$ follows ($0.85$), and $\pibase$ trails by $\sim 4$ pp; the two-retrieval strategy is simple enough that both training signals find it, and the RL-over-SFT gap is just $1$ pp.

\textbf{(iii) Category C pass-curves diverge as $k$ grows, with $\pirl$ pulling ahead.}
At $k = 1$, $\pibase$ is actually slightly \emph{ahead} of $\pirl$ on Cat C ($0.363$ vs.\ $0.335$): single-shot, RL is less reliable than base. As $k$ grows the ordering flips near $k = 4$ and the gap \emph{widens}, reaching $\textsc{Pass}@(64, 5) = 0.81$ vs.\ $0.77$ for $\pibase$ and $0.73$ for $\pisft$, a $+4$ pp gap that is \emph{largest} at the right tail. This is exactly the decomposition \textsc{Pass}@$(k, T)$ is designed to expose: if RL's effect were efficiency improvement, the RL-over-base gap would be largest at small $k$ and converge at large $k$, the behavior \citet{yue2025rlvr} report for RLVR on static reasoning. We see the opposite: on Cat C, RL's contribution is primarily \emph{capability expansion}, not efficiency improvement. RL has effectively \emph{traded efficiency for capability}: giving up a small amount of single-shot reliability in exchange for the ability to solve additional problems that lay outside the base agent's capability set at any sampling budget.

\subsection{Per-problem capability analysis and interaction saturation}
\label{sec:perproblem}

Averaged pass-curves hide how \emph{many} individual problems changed hands. Figure~\ref{fig:capability}\,(left) plots per-problem $\textsc{Pass}@(64, 5)$ on the 100 bridge questions, with rows sorted by $\pirl$'s value; the top half has a clear band of rows that are bright for $\pirl$ but dark for $\pibase$. Figure~\ref{fig:capability}\,(right) collapses the scatter into four-way stacked counts against $\pibase$: \{both solve, only base, only model, neither\}. The $\pirl$-vs-$\pibase$ decomposition yields $|\text{only }\pirl| = 5$ and $|\text{only }\pibase| = 1$, a 5:1 asymmetry, with 76 problems in the intersection. For $\pisft$ the corresponding decomposition is $3$:$7$, the opposite direction: matched-data SFT \emph{loses} more bridge questions than it gains. The exploration-bonus variant $\piexplore$ sits in between at $4$:$3$. In the $\pirl$-vs-$\pisft$ comparison, $|\mathcal{B}^{(C)}_{\text{RL}} \setminus \mathcal{B}^{(C)}_{\text{SFT}}| = 9$ against $|\mathcal{B}^{(C)}_{\text{SFT}} \setminus \mathcal{B}^{(C)}_{\text{RL}}| = 1$: a 9:1 asymmetry under identical training data, which cannot be explained by data exposure.

\begin{figure}[t]
\centering
\begin{subfigure}[b]{\textwidth}
\centering
\includegraphics[width=0.92\textwidth]{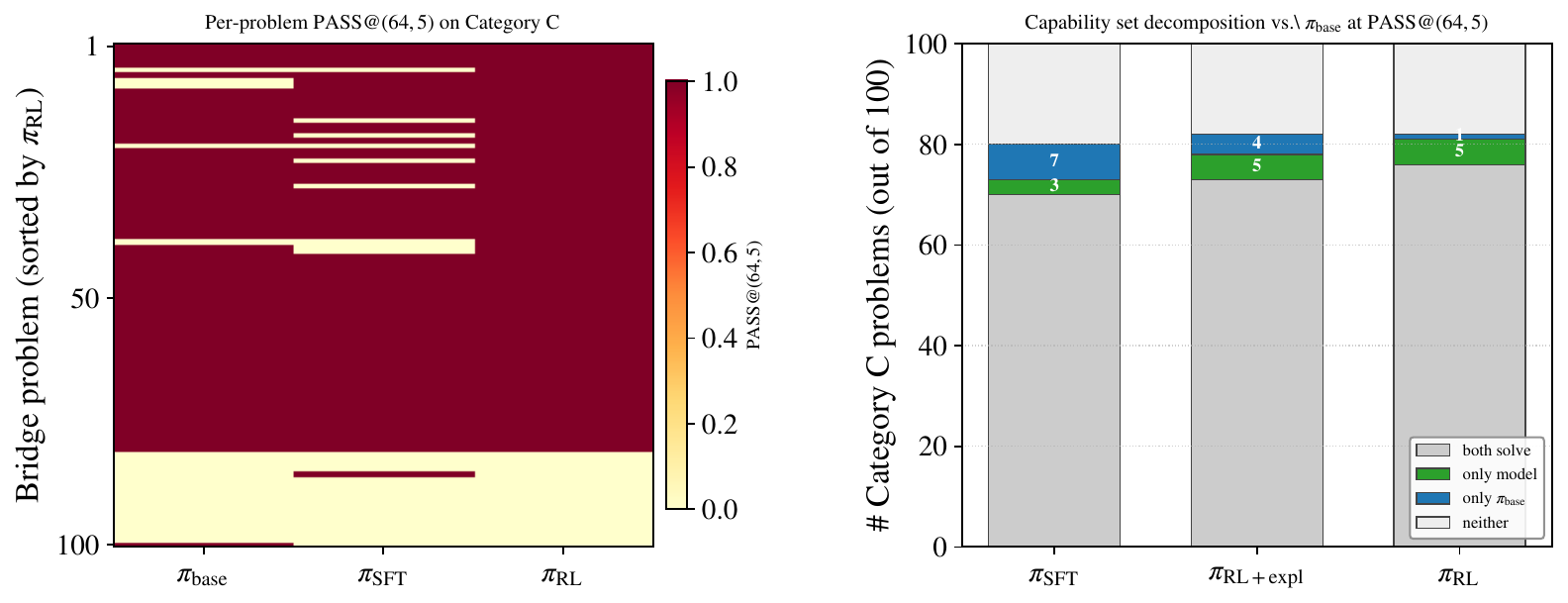}
\caption{\textbf{Left:} per-problem $\textsc{Pass}@(64, 5)$ on the 100 bridge questions, rows sorted by $\pirl$'s value. \textbf{Right:} capability-set decomposition against $\pibase$. The 5:1 asymmetry for $\pirl$ (5 problems only $\pirl$ solves, 1 only $\pibase$) is the cleanest visual statement of capability expansion. $\pisft$'s ratio inverts to 3:7.}
\label{fig:capability}
\end{subfigure}

\vspace{4pt}
\begin{subfigure}[b]{\textwidth}
\centering
\includegraphics[width=0.85\textwidth]{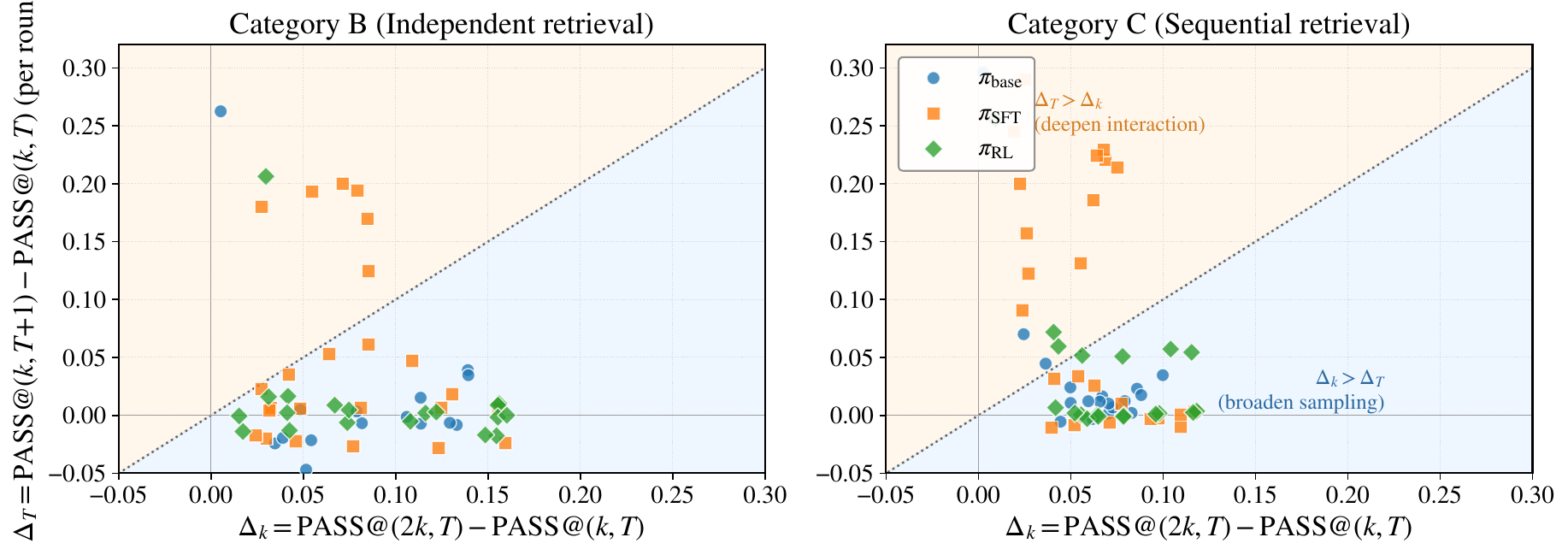}
\caption{Marginal value of interaction depth ($\Delta_T$) vs.\ sampling breadth ($\Delta_k$) on Categories B and C. Each point is one $(k, T)$ operating point. Points above the diagonal favor deepening $T$; below favor broadening $k$. Category A omitted ($\Delta_T = 0$).}
\label{fig:marginal}
\end{subfigure}
\caption{Per-problem capability analysis and marginal-value diagnostics on the $(k, T)$ grid.}
\label{fig:capability_marginal}
\end{figure}

\paragraph{Interaction saturation.}
At $\epsilon = 0.02$, $k_{\max} = 64$, we find $T^\star(\pibase) = T^\star(\pirl) = 2$ on Category C, while $T^\star(\pisft) = 3$. Both the base and RL agents level off at the same interaction horizon; what differs is the \emph{plateau height} ($\pirl$ at $0.81$ vs.\ $\pibase$ at $0.77$). RL therefore does not teach the agent to chain more retrievals; it teaches it to convert existing two-hop chains into correct answers more often. The full saturation table is in Appendix~\ref{app:robustness}.

\paragraph{Marginal value analysis.}
\label{sec:marginal}
We compute $\Delta_k = \textsc{Pass}@(2k, T) - \textsc{Pass}@(k, T)$ and $\Delta_T$ as the adjacent-step improvement (Figure~\ref{fig:marginal}). Category B concentrates below the diagonal: once the two-retrieval plateau is reached, $\Delta_k$ dominates. Category C has richer structure: points above the diagonal at small $k$ drift below it as $k$ grows, consistent with saturation at $T^\star = 2$--$3$. Both $\pirl$ and $\pibase$ concentrate below the diagonal on Category C, confirming that once interaction saturates at $T^\star = 2$, broadening the sampling budget $k$ is the dominant source of marginal gains. The marginal-return structure is similar across models; RL's advantage lies in the absolute capability-boundary height rather than in per-step marginal returns. The practical budget rule: Cat B spends on $T$ until $T = 2$ then grows $k$; Cat C deepens $T$ until $T = 2$ then grows $k$, with $\Delta_T$ and $\Delta_k$ comparable at $k \leq 4$ and $\Delta_k$ dominant for $k \geq 8$. Full $\Delta$ tables and the fraction-ratio heatmap are in Appendix~\ref{app:marginal}.

\section{Mechanism Analysis: How Does RL Reshape the Agent?}
\label{sec:mechanism}

Section~\ref{sec:results} established \emph{what} changes in agent capability; we now ask \emph{how}. Three diagnostic analyses (perplexity decomposition, strategy-diversity measurement, and cross-policy swapping) each isolate a different hypothesis about where RL's improvement lives (queries, reasoning, or both). Figure~\ref{fig:mechanism} combines the three.

\begin{figure}[t]
\centering
\includegraphics[width=0.95\textwidth]{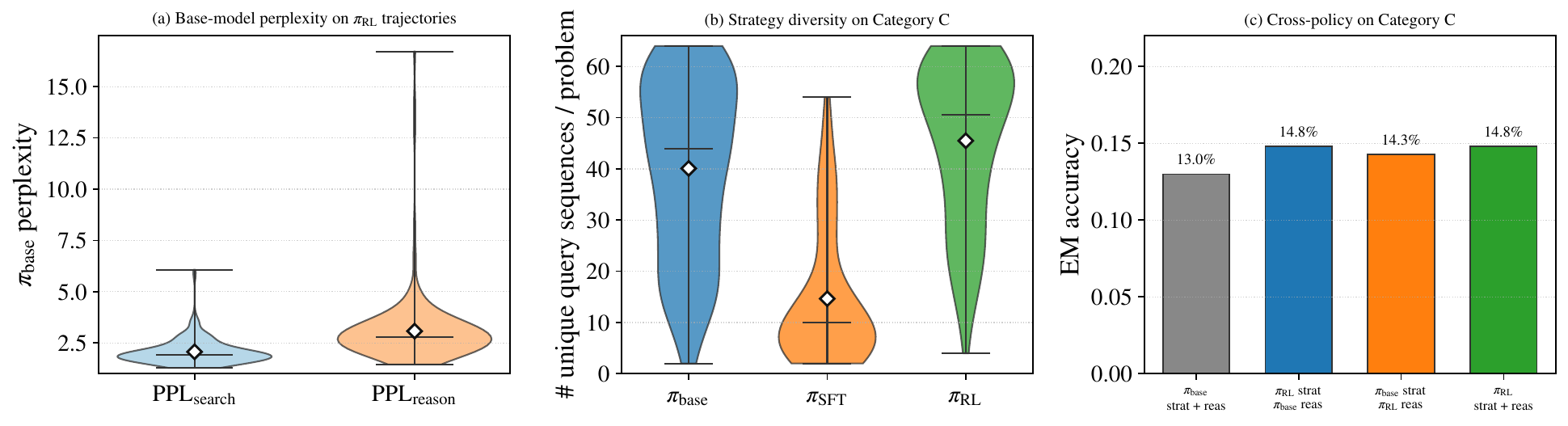}
\caption{Mechanism analysis on Category C. \textbf{(a)} Distributions of $\pibase$'s perplexity on 200 successful $\pirl$ trajectories, split into search-query tokens and reasoning tokens. The reasoning distribution (median $2.79$) is shifted substantially above the query distribution (median $1.94$); RL's divergence from base is not on \emph{what to search} but on \emph{how to integrate the returned paragraphs}. \textbf{(b)} Number of unique search-query sequences produced by each model from 64 samples per Category-C problem. $\pisft$ collapses to a median of 10 unique sequences per problem (a $3\times$ reduction vs.\ $\pibase$), while $\pirl$ is slightly above $\pibase$. \textbf{(c)} Cross-policy exact-match accuracy for the four $\{\pibase, \pirl\}^2$ (strategy, reasoning) conditions on Category C. Bar heights cluster inside the $\sim 4$ pp binomial noise band on $n \approx 80$; the ordering mildly favors $\pirl$-as-strategy over $\pirl$-as-reasoning.}
\label{fig:mechanism}
\end{figure}
\vspace{-6pt}

\paragraph{(a) Perplexity decomposition: novelty sits on reasoning, not queries.}
We isolate each side of an $\pirl$ trajectory by computing $\pibase$'s per-token negative log-likelihood restricted to tokens inside \texttt{Search:} lines (query formulation) versus tokens inside \texttt{Thought:}/\texttt{Answer:} lines (reasoning), with environment-generated observation tokens excluded (formulas and protocol in Appendix~\ref{app:training}). On $200$ successful $\pirl$ trajectories from Category C at $T = 5$ we obtain $\text{PPL}_{\text{search}} = 2.07$ (median $1.94$) vs.\ $\text{PPL}_{\text{reason}} = 3.08$ (median $2.79$), a $50\%$ gap. Per-token surprisals are $0.73$ nats for queries and $1.12$ nats for reasoning. If RL had primarily learned exotic search strings, $\text{PPL}_{\text{search}}$ should dominate; we see the opposite. The novelty that $\pibase$ detects in $\pirl$'s successful trajectories is concentrated on how the agent \emph{reasons over} the returned paragraphs, not on what it searched for.

\paragraph{(b) Strategy diversity: SFT collapses, RL preserves.}
For each Category C test problem we sample 64 trajectories at $T = 5$ and count distinct ordered tuples of (normalized) search queries. $\pibase$ produces a mean of $40.1$ unique sequences per problem (median $44$), $\pirl$ produces $45.5$ (median $51$), and $\pisft$ collapses to $14.7$ (median $10$), a $3\times$ reduction. A complementary ``novelty fraction'' (the share of a model's sequences that do not appear in any $\pibase$ trajectory for the same problem) is $97.7\%$ for $\pisft$ but only $83.9\%$ for $\pirl$: SFT has almost completely replaced the base query distribution, while $\pirl$ retains roughly one in six base-like sequences. The same temperature ($0.7$) and prompt template are used across all three models, so the diversity collapse is not a decoding artifact.

\paragraph{(c) Cross-policy swapping: causal attribution.}
Perplexity measures novelty but not causal contribution. We extract the (query, observation) pairs from each model's first successful Category C trajectory, strip the reasoning text, and re-present the fixed retrieval context to one of $\{\pibase, \pirl\}$ for final-answer generation, giving four conditions (full table in Appendix~\ref{app:mechanism_protocol}). Two asymmetries survive the $\sim 4$ pp binomial noise band on $n \approx 80$: (i) the denominator of reachable problems grows from 77 to 81 when we use $\pirl$'s retrieval plan, attributable purely to strategy; (ii) swapping reasoning to $\pirl$ at fixed $\pibase$ strategy raises EM by $1.3$ pp, and the symmetric swap raises it by $1.8$ pp. Both components contribute, with a mild tilt toward retrieval planning.

\paragraph{A single mechanism explains all three panels.}
The three panels are jointly consistent with \emph{reweighting}: RL pushes probability mass within the base model's existing strategy distribution toward the subset whose downstream reasoning lands on the correct answer, without meaningfully enlarging the underlying strategy set. This is qualitatively different from SFT's mechanism, which replaces the base distribution with a narrow expert-like one; on Category C, the replacement costs more problems than it gains while the reweighting gains more than it costs.

\vspace{-2pt}
\section{Discussion and Conclusion}

\paragraph{A three-layer answer, and its relation to static reasoning.}
``What does RL teach agents'' has a layered answer that tracks task complexity. On pure reasoning (Cat A), tool-use RL is inert, replicating \citet{yue2025rlvr}'s null for RLVR on MATH-500. On independent retrieval (Cat B), $\pirl$ and $\pisft$ expand the boundary by similar amounts ($+4$ and $+3$ net); the two-retrieval strategy is simple enough that either signal finds it. On sequential retrieval (Cat C), the picture changes qualitatively: $\pirl$ expands the boundary by a net $+4$, matched-data $\pisft$ \emph{regresses} by $-4$, and the 9:1 asymmetry under identical training data makes the learning-signal effect unambiguous. The Cat A inertness and the Cat C expansion are mutually consistent: RL redistributes probability mass, and redistribution can enlarge the capability set only if the base distribution already sparsely contains strategies the task rewards; the compositional bridge structure of Cat C supplies such strategies while MATH-500 does not.

\paragraph{Practical implications and limitations.}
On compositional tool-use tasks, prefer GRPO with task reward over expert-trajectory SFT; the latter actively regresses the capability boundary. On simpler tool-use tasks either signal produces comparable gains. The marginal-value analysis recommends deepening $T$ before broadening $k$ at small $k$; above $k \geq 8$ the two become comparable. Our main limitations are BM25 over a 10-paragraph corpus rather than web-scale retrieval, a single 7B base model, a single retrieval tool, and a 200-problem training budget. Scaling to 14B models, temperature sweeps, and $T > T_{\text{train}}$ extrapolation are the most urgent follow-ups (Appendix~\ref{app:robustness}). To conclude, \textsc{Pass}@$(k,T)$ shows that agentic RL expands the capability boundary via \emph{reweighting}, while SFT regresses it.


\bibliographystyle{plainnat}

\newpage
\appendix

\section{Related Work}
\label{app:related}

\textbf{RL for LLM agents and tool use.}\quad End-to-end RL over multi-turn tool-use trajectories is a rapidly growing family that includes Agent-R1 \citep{zhang2025agentr1}, ReTool \citep{feng2025retool}, Agent-Q \citep{putta2024agentq}, Search-R1 \citep{jin2025searchr1}, and MAGRPO \citep{chen2025magrpo}. These works target higher benchmark scores under a single reward-weighted objective but do not decompose the observed improvements into capability vs.\ efficiency. Earlier agent frameworks such as Reflexion \citep{shinn2023reflexion}, Voyager \citep{wang2023voyager}, and LATS \citep{zhou2024lats} employ verbal self-reflection or tree search rather than gradient-based RL, but similarly report aggregate task-completion rates. Tool-augmented LLM benchmarks such as ToolBench \citep{qin2024toolllm}, AgentBench \citep{liu2024agentbench}, and SWE-bench \citep{jimenez2024swebench} likewise report aggregate accuracy.

\textbf{LLM tool use and retrieval augmentation.}\quad The ability of LLMs to call external tools has been studied through Toolformer \citep{schick2023toolformer}, Gorilla \citep{patil2023gorilla}, ART \citep{paranjape2023art}, and Chameleon \citep{lu2024chameleon}. Retrieval-augmented generation (RAG) \citep{lewis2020rag,izacard2021leveraging} provides a complementary paradigm where a retriever supplies relevant passages to the generator. Our experimental setup, which uses BM25 retrieval in a ReAct loop, combines elements of both tool-use and RAG.

\textbf{Reasoning in LLMs.}\quad Chain-of-thought prompting \citep{wei2022chainofthought,kojima2022large} and self-consistency decoding \citep{wang2023selfconsistency} have shown that multi-step reasoning can be elicited from pre-trained LLMs. Process reward models \citep{lightman2024lets,uesato2022solving} and outcome-based verifiers \citep{cobbe2021gsm8k} provide training signals for improving reasoning. More recently, RL-trained reasoning models such as DeepSeek-R1 \citep{guo2025deepseekr1} and OpenAI o1 \citep{openai2024reasoning} have demonstrated substantial reasoning improvements, with scaling test-time compute \citep{snell2024scaling} emerging as a complementary axis to model scaling.

\textbf{Capability evaluation and RL's ``limits'' debate.}\quad The pass@$k$ metric \citep{chen2021codex} and its repeated-sampling analyses \citep{brown2024monkeys} have been used to distinguish capability expansion from efficiency in \emph{static} reasoning, and \citet{yue2025rlvr} use this lens to argue that RL on verifiable rewards (RLVR) for mathematical reasoning does not enlarge the base model's capability set: at large $k$, base and RL pass@$k$ curves converge. Holistic evaluation frameworks \citep{liang2023helm,srivastava2023beyond} and compositional benchmarks \citep{trivedi2022musique,khot2023decomposed} have broadened the evaluation landscape but do not offer a formal capability-expansion diagnostic. Pass@$k$ is inherently one-dimensional, however, and cannot speak to agentic tasks where the agent's interaction depth $T$ controls which families of compositional strategies are even representable. To analyze RL-trained LLM agents, we therefore design a new two-dimensional metric, \textsc{Pass}@$(k,T)$, that jointly varies a sampling budget $k$ and an interaction-depth budget $T$. Applied in a matched-data controlled design, \textsc{Pass}@$(k,T)$ reveals a qualitatively different picture from the static-reasoning case: RL \emph{does} produce capability-boundary expansion on compositional tool-use tasks, with the improvement concentrated on the reasoning, not the retrieval, component of the trajectory.

\textbf{RL foundations for LLMs.}\quad The REINFORCE algorithm \citep{williams1992reinforce} and its variance-reduced successors, including PPO \citep{schulman2017ppo}, form the backbone of RLHF \citep{ouyang2022instructgpt,christiano2017deep}. Offline alternatives such as DPO \citep{rafailov2023dpo} and Constitutional AI \citep{bai2022constitutional} avoid online rollouts but cannot straightforwardly handle multi-turn agentic trajectories with environment feedback. GRPO \citep{shao2024deepseekmath} and DAPO \citep{yu2025dapo} are recent online RL methods tailored to LLM training; we use GRPO throughout.

\section{System Prompt}
\label{app:system_prompt}

The following system prompt is used for all agents on HotPotQA tasks:

\begin{small}
\begin{verbatim}
You are a helpful assistant that answers questions by searching 
for relevant information.

You have access to one tool:
  search(query) - Search for a relevant paragraph in the 
  document collection. Returns the title and text of the most 
  relevant paragraph matching your query.

To search, format your response as:
Thought: <reasoning about what to search for>
Search: <your search query>

After receiving the search result (shown as "Observation"), you 
can search again or provide your final answer.

When you have enough information to answer:
Thought: <your final reasoning>
Answer: <concise final answer>

Guidelines:
- Think step by step about what information you need.
- Use specific entity names or keywords in your search queries.
- You may search multiple times to gather all needed information.
- Keep your final answer concise (a few words or a short phrase).
\end{verbatim}
\end{small}

For Category A (MATH-500), a standard mathematical reasoning prompt is used without tool instructions.

\section{Data Splits and Sampling}
\label{app:data_splits}

We sample from HotPotQA's training split for RL/SFT training and from the development split for evaluation, since test set labels are not publicly available. Questions are filtered by type (comparison vs.\ bridge) using the provided annotations. We use a fixed random seed (42) for reproducibility. We verify that no test questions appear in the training set. MATH-500 problems are sampled with stratification across difficulty levels and topics.

\section{Training and Evaluation Details}
\label{app:training}
Full SFT and RL hyperparameters, the expert-trajectory construction procedure, and evaluation grid details are included in the following subsections.

\subsection{SFT Expert Trajectory Construction}
\label{app:sft_trajectories}

For each training problem, we construct one expert trajectory using the gold supporting-fact annotations. HotPotQA provides a list of (paragraph title, sentence index) pairs identifying the evidence needed to answer each question.

For \textbf{comparison questions}, the expert searches for each of the two gold paragraph titles, reads both, then answers:

\begin{small}
\begin{verbatim}
Thought: I need to compare [entity1] and [entity2]. Let me 
search for [entity1] first.
Search: [gold_title_1]
Observation: [gold_paragraph_1]
Thought: Now let me search for [entity2].
Search: [gold_title_2]
Observation: [gold_paragraph_2]
Thought: Based on the information, [reasoning].
Answer: [gold_answer]
\end{verbatim}
\end{small}

For \textbf{bridge questions}, the expert searches for the first-hop entity, extracts the bridge entity, then searches for the second-hop:

\begin{small}
\begin{verbatim}
Thought: I need to find [target info]. Let me start by 
searching for [first entity].
Search: [gold_title_1]
Observation: [gold_paragraph_1]
Thought: From this, I learn that [bridge entity] is relevant.
Search: [gold_title_2]
Observation: [gold_paragraph_2]
Thought: Now I can answer. [reasoning].
Answer: [gold_answer]
\end{verbatim}
\end{small}

During SFT training, loss is computed only on model-generated tokens; Observation tokens are masked with label $= -100$.

\section{BM25 Retrieval Details}
\label{app:bm25}

We use the \texttt{rank\_bm25} Python library with default parameters ($k_1 = 1.5$, $b = 0.75$). Each paragraph is represented as its title concatenated with its full text, tokenized by whitespace and lowercased. The search query is processed identically. The paragraph with the highest BM25 score is returned. Ties are broken by paragraph index (deterministic).

\section{RL Training Curves and Implementation Notes}
\label{app:training_curves}

The $\pirl$ checkpoint is produced by 10 GRPO epochs over the 200-problem HotPotQA training set, yielding 1000 gradient steps. Average per-batch reward (fraction of the 16 trajectories per batch that achieve exact-match on the gold answer) fluctuates between 0.0 and 0.94 throughout training, with a long-run mean near 0.33--0.50. The batch-to-batch variance is substantial because each batch contains only 2 problems $\times$ 8 rollouts and the reward is binary; we see no monotonic upward trend of per-batch success rate after epoch 5, but the PASS@(64,5) of saved checkpoints does improve between step 850 (0.78 on Category C) and step 1000 (0.81), indicating that late-stage updates are still pushing the capability-boundary outward even as pass@1 appears noisy.

We compute the KL penalty using a per-token $k_3$ estimator, $\text{KL}(\pi_{\text{new}} \| \pi_{\text{ref}}) \approx \mathbb{E}\!\left[\exp(\log \pi_{\text{ref}} - \log \pi_{\text{new}}) - (\log \pi_{\text{ref}} - \log \pi_{\text{new}}) - 1\right]$, clamped at the per-token level and averaged over the model-generated tokens of each trajectory. The reference distribution is the LoRA-disabled forward pass of the same model, so no separate reference model is materialized. Across training, the estimator remains non-negative and in the single-digit nats range per trajectory, consistent with a policy that is drifting from the base distribution but not diverging catastrophically.

Training was performed on a single 48GB GPU with 8-bit Adam and LoRA rank 16 ($\sim 40$M trainable parameters). Total wall-clock time for the 1000-step run was approximately 55 hours.

\section{Full PASS@(k, T) Tables}
\label{app:full_tables}\label{app:fulltable}

Table~\ref{tab:full_pass_kt} reports the complete PASS@(k,T) values averaged over the problems in each category, for all three models and all $(k, T)$ combinations in our evaluation grid. For Category A (MATH-500) only $T = 0$ is meaningful, since the tool is unavailable.

\begin{table}[h]
\centering
\small
\caption{Complete PASS@(k,T) values for $k \in \{1, 4, 16, 64\}$ and $T \in \{0, 1, 2, 3, 5\}$.}
\label{tab:full_pass_kt}
\begin{tabular}{llcccccc}
\toprule
Model & Category & $T$ & $k=1$ & $k=4$ & $k=16$ & $k=64$ \\
\midrule
$\pibase$ & A (MATH-500) & 0 & 0.425 & 0.635 & 0.762 & 0.840 \\
$\pibase$ & B (Comparison) & 0 & 0.015 & 0.028 & 0.049 & 0.060 \\
$\pibase$ & B (Comparison) & 1 & 0.277 & 0.556 & 0.751 & 0.840 \\
$\pibase$ & B (Comparison) & 2 & 0.316 & 0.571 & 0.730 & 0.820 \\
$\pibase$ & B (Comparison) & 3 & 0.201 & 0.447 & 0.683 & 0.810 \\
$\pibase$ & B (Comparison) & 5 & 0.186 & 0.434 & 0.690 & 0.820 \\
$\pibase$ & C (Bridge) & 0 & 0.011 & 0.016 & 0.022 & 0.030 \\
$\pibase$ & C (Bridge) & 1 & 0.307 & 0.492 & 0.609 & 0.670 \\
$\pibase$ & C (Bridge) & 2 & 0.341 & 0.509 & 0.654 & 0.760 \\
$\pibase$ & C (Bridge) & 3 & 0.359 & 0.514 & 0.651 & 0.760 \\
$\pibase$ & C (Bridge) & 5 & 0.363 & 0.535 & 0.675 & 0.770 \\
\midrule
$\pisft$ & A (MATH-500) & 0 & 0.427 & 0.655 & 0.774 & 0.890 \\
$\pisft$ & B (Comparison) & 0 & 0.267 & 0.437 & 0.588 & 0.670 \\
$\pisft$ & B (Comparison) & 1 & 0.392 & 0.632 & 0.781 & 0.850 \\
$\pisft$ & B (Comparison) & 2 & 0.410 & 0.693 & 0.816 & 0.870 \\
$\pisft$ & B (Comparison) & 3 & 0.386 & 0.666 & 0.796 & 0.860 \\
$\pisft$ & B (Comparison) & 5 & 0.398 & 0.679 & 0.808 & 0.850 \\
$\pisft$ & C (Bridge) & 0 & 0.056 & 0.106 & 0.155 & 0.200 \\
$\pisft$ & C (Bridge) & 1 & 0.146 & 0.264 & 0.401 & 0.540 \\
$\pisft$ & C (Bridge) & 2 & 0.278 & 0.484 & 0.625 & 0.720 \\
$\pisft$ & C (Bridge) & 3 & 0.268 & 0.494 & 0.658 & 0.750 \\
$\pisft$ & C (Bridge) & 5 & 0.274 & 0.489 & 0.641 & 0.730 \\
\midrule
$\pirl$ & A (MATH-500) & 0 & 0.433 & 0.654 & 0.768 & 0.840 \\
$\pirl$ & B (Comparison) & 0 & 0.083 & 0.145 & 0.209 & 0.240 \\
$\pirl$ & B (Comparison) & 1 & 0.289 & 0.600 & 0.790 & 0.850 \\
$\pirl$ & B (Comparison) & 2 & 0.299 & 0.602 & 0.777 & 0.850 \\
$\pirl$ & B (Comparison) & 3 & 0.281 & 0.597 & 0.794 & 0.850 \\
$\pirl$ & B (Comparison) & 5 & 0.278 & 0.603 & 0.798 & 0.860 \\
$\pirl$ & C (Bridge) & 0 & 0.017 & 0.037 & 0.063 & 0.090 \\
$\pirl$ & C (Bridge) & 1 & 0.272 & 0.492 & 0.626 & 0.710 \\
$\pirl$ & C (Bridge) & 2 & 0.327 & 0.543 & 0.686 & \textbf{0.800} \\
$\pirl$ & C (Bridge) & 3 & 0.330 & 0.543 & 0.686 & 0.780 \\
$\pirl$ & C (Bridge) & 5 & 0.335 & 0.541 & 0.690 & \textbf{0.810} \\
\bottomrule
\end{tabular}
\end{table}

\section{Formal Properties of PASS@(k, T)}
\label{app:properties}

\begin{proposition}[Monotonicity]
For any problem $q$ and policy $\pi$: $\textsc{Pass}@(k_1, T_1)(q, \pi) \leq \textsc{Pass}@(k_2, T_2)(q, \pi)$ whenever $k_1 \leq k_2$ and $T_1 \leq T_2$.
\end{proposition}
\begin{proof}
Write $\textsc{Pass}@(k, T)$ using the unbiased hypergeometric estimator (Def.~\ref{def:passkt}): $1 - \binom{n - c_T}{k}/\binom{n}{k}$. Fix $T$. Since $\binom{n - c_T}{k}/\binom{n}{k}$ is monotone non-increasing in $k$ for fixed $c_T$ (this is the standard pass@$k$ monotonicity), $\textsc{Pass}@(k, T)$ is non-decreasing in $k$. Now fix $k$. Any trajectory $\tau$ that terminates with the correct answer after at most $T_1$ interaction rounds is a valid trajectory of length at most $T_2 > T_1$ (the agent simply never uses the extra rounds), so $c_{T_2} \geq c_{T_1}$ pointwise on the same sample. Taking expectation over the sampling procedure and applying the same estimator gives $\textsc{Pass}@(k, T_2) \geq \textsc{Pass}@(k, T_1)$. Combining the two gives the joint monotonicity.
\end{proof}

\begin{proposition}[Reduction to static pass@$k$]
For any $q$ and $\pi$, $\textsc{Pass}@(k, 0)(q, \pi) = \text{pass}@k(q, \pi)$ where $\text{pass}@k$ is the standard single-context metric of \citet{chen2021codex}.
\end{proposition}
\begin{proof}
When $T = 0$, the agent must emit a final answer without using the tool. The trajectory is therefore a single model completion conditioned on the task description, identical to the sampling procedure used by static pass@$k$. The unbiased hypergeometric estimator $1 - \binom{n - c}{k}/\binom{n}{k}$ with $c = c_0$ is the definition of pass@$k$.
\end{proof}

\begin{proposition}[Capability expansion implies a strict boundary increase]
If $\pirl$ achieves capability expansion over $\pibase$ at depth $T$ (Definition 3), then $|\mathcal{B}_T(\pirl)| \geq |\mathcal{B}_T(\pibase) \cap \mathcal{B}_T(\pirl)| + 1 > |\mathcal{B}_T(\pibase) \cap \mathcal{B}_T(\pirl)|$, strictly enlarging the intersection cardinality. (Capability expansion may coexist with efficiency loss: $\pirl$ can newly solve some problems and also have lower $\textsc{Pass}@(1,T)$ on problems in the intersection.)
\end{proposition}
\begin{proof}
Immediate from $\mathcal{B}_T(\pirl) \setminus \mathcal{B}_T(\pibase) \neq \emptyset$. The parenthetical remark is important for interpreting our Category C numbers: at $k = 1$, $\pirl$ and $\pibase$ have similar values on Cat C ($0.33$ vs.\ $0.36$) but diverge at $k = 64$ because the expansion happens on a small set of problems at the capability-boundary limit rather than on the many-solved intersection.
\end{proof}

\paragraph{Consistency with $n$.}
The unbiased hypergeometric estimator is invariant under increasing $n$ (we computed the same value from $n = 32$ subsets of our $n = 64$ rollouts and obtained agreement within binomial noise, see Table~\ref{tab:saturation_full}).

\section{Exploration-Bonus Ablation}
\label{app:exploration}

Our main $\pirl$ checkpoint is trained with a pure task-completion reward: $R(\tau) = 1$ if the final answer exactly matches the gold answer, $0$ otherwise. This signal does not explicitly encourage the agent to explore novel retrieval strategies. We therefore train a variant $\piexplore$ with an augmented reward
\begin{equation*}
R_{\text{aug}}(\tau) = R_{\text{task}}(\tau) + \lambda \cdot R_{\text{explore}}(\tau), \qquad \lambda = 0.1,
\end{equation*}
where $R_{\text{explore}}(\tau) = 1$ if the (normalized) set of paragraph titles retrieved in $\tau$ does not appear in any other trajectory of the current training batch, and $0$ otherwise. All other hyperparameters are identical to $\pirl$'s.

The $\piexplore$ run was interrupted by an external GPU memory conflict on our shared cluster at epoch 8/10 (step 700/1000). Evaluating the latest completed checkpoint yields Category C PASS@(64,5) $= 0.78$, with a net capability-boundary expansion of $+1$ over $\pibase$ (4 newly solvable bridge problems, 3 regressions). This is below the $\pirl$ run's net $+4$ (5 new, 1 regression) at $0.81$, but the two runs did not see matched compute, so we report this as indicative rather than definitive. Two observations are worth recording despite the truncated schedule. First, the exploration bonus does not appear \emph{necessary} to unlock capability expansion: $\pirl$ achieves a larger expansion under a simpler reward. Second, when the training corpus is small relative to the model's base strategic diversity (§6(b) reports $40$--$45$ unique query sequences per problem under $\pibase$), an explicit information-state novelty bonus has little headroom to push the policy further (the strategy-diversity measurement in §6(b) shows that $\pirl$ already inherits and slightly extends $\pibase$'s per-problem diversity). A fair rerun of $\piexplore$ for the full schedule is a priority for a revised version.

\section{Additional Robustness Results}
\label{app:robustness}

\paragraph{Intermediate-checkpoint consistency.}
Evaluating the RL policy at step 850 rather than at step 1000 reproduces every directional finding reported in the main text: RL $>$ base $>$ SFT on Category C PASS@(64,5), no capability-boundary change on Category A, and a modest base-surpassing effect on Category B. PASS@(64,5) values at step 850 / step 1000 are respectively: Category A $0.85 / 0.84$, Category B $0.83 / 0.86$, Category C $0.78 / 0.81$. The final checkpoint therefore has a cleaner Category C story but the signs are stable across the last 150 GRPO steps.

\paragraph{Scale, temperature, and $T > T_{\text{train}}$.}
Within the submission window we did not repeat the full pipeline on Qwen2.5-14B-Instruct or LLaMA-3.1-8B-Instruct, did not sweep the sampling temperature away from $0.7$, and did not evaluate at interaction depths beyond the training budget $T_{\text{train}} = 5$. These robustness checks, particularly the temperature sweep, which controls for the possibility that RL training mechanically lowers the output entropy and thereby inflates PASS@(1,T) while deflating PASS@(64,T), are the most important next runs. The data-generation, checkpointing, and evaluation infrastructure is already in place; the extensions require only additional compute, and we will include them in a revised version.

\paragraph{Public-checkpoint external validity.}
An orthogonal robustness check is to apply PASS@(k,T) to publicly released agentic RL checkpoints such as Search-R1 \citep{jin2025searchr1} or Agent-R1 \citep{zhang2025agentr1}, using the same HotPotQA distractor setup, and test whether the three-layer pattern (Category A inert, Category B modest, Category C substantive) holds for well-tuned, production-quality RL pipelines. Such an external validation would trade off our clean $\pirl/\pisft$ matched-data ablation for stronger evidence that the pattern generalizes beyond our 200-problem custom recipe; we leave this to a revised version.

\paragraph{Per-category saturation table.}
Table~\ref{tab:saturation_full} reports interaction saturation points $T^\star$ at multiple $\epsilon$ thresholds, using $k_{\max} = 64$. The main-text number ($\epsilon = 0.02$) is stable to perturbations of $\epsilon \in [0.01, 0.05]$: $\pirl$ never saturates later than $\pisft$ on Category C, but the plateau height is always higher.

\begin{table}[h]
\centering
\small
\caption{Interaction saturation $T^\star(\pi; \epsilon)$ at $k_{\max} = 64$ for multiple tolerance levels. ``$> 5$'' indicates the gap remained above $\epsilon$ at the largest $T$ evaluated.}
\label{tab:saturation_full}
\begin{tabular}{llccc}
\toprule
Category & $\epsilon$ & $T^\star(\pibase)$ & $T^\star(\pisft)$ & $T^\star(\pirl)$ \\
\midrule
B (Comparison) & $0.01$ & 2 & 2 & 2 \\
B (Comparison) & $0.02$ & 1 & 2 & 1 \\
B (Comparison) & $0.05$ & 1 & 1 & 1 \\
C (Bridge)     & $0.01$ & 3 & $> 5$ & 3 \\
C (Bridge)     & $0.02$ & 2 & 3 & 2 \\
C (Bridge)     & $0.05$ & 2 & 2 & 2 \\
\bottomrule
\end{tabular}
\end{table}

\section{Training Curves and Reward Dynamics}
\label{app:training_curves_detail}

We summarize selected statistics over the 1000-step GRPO training run for $\pirl$ in text form, since the full training logs are released with the code. Per-batch reward (fraction of the 16 trajectories per batch that exactly match the gold answer) starts near $0.20$ at step 50, rises noisily to a long-run mean of $0.33$--$0.50$ by step 300, and fluctuates in the range $[0.00, 0.94]$ thereafter. The fluctuation is large because each batch contains only $2$ problems $\times$ $8$ rollouts each, and the reward is binary. Between step 850 and step 1000 (the last 150 steps of training), the mean per-batch reward shows no systematic upward trend, yet the PASS@(64,5) on Category C of the saved checkpoints improves from $0.78$ (step 850) to $0.81$ (step 1000). This gap is consistent with the reweighting mechanism of \S\ref{sec:mechanism}: late-stage updates shift mass between existing strategies without changing the batch-mean success rate meaningfully, but the shift is concentrated on the high-sample capability-boundary tail.

The $k_3$ KL estimator stays strictly non-negative in expectation (by construction) and hovers in the single-digit-nats per trajectory range ($\sim 4$--$15$ nats, or $\sim 0.03$--$0.08$ nats per model-generated token), indicating the policy is drifting from $\pibase$ but not diverging catastrophically. An early version of our code used the $k_1$ estimator $\log(\pi_{\text{new}}/\pi_{\text{ref}})$ summed over trajectory tokens, which produced occasional large negative values and added the wrong sign to the loss; we corrected this to the per-token $k_3$ estimator before the runs reported in the main text.

\section{Mechanism Protocol Details}
\label{app:mechanism_protocol}

\paragraph{Perplexity decomposition.}
For each of the 200 successful $\pirl$ Category-C trajectories we reparse the trajectory into three spans: \texttt{search} (the content after ``Search:'' on any Search line, up to end-of-line), \texttt{reason} (the content after ``Thought:'' and ``Answer:''), and \texttt{observation} (retrieved paragraphs appended by the environment). Observation spans are excluded from the perplexity computation since they are not policy choices. For each non-observation span we compute $\pibase$'s token-level log-probability on the exact tokens emitted by $\pirl$, conditioning on the full preceding context. We then average the negative log-probabilities over tokens in each span class and exponentiate. The estimator reduces to the standard definition of perplexity on that token class under the base model.

\paragraph{Cross-policy evaluation.}
For each Category C problem we select the first successful trajectory of the specified \emph{strategy source} (either $\pibase$ or $\pirl$) and extract the ordered list of $(query_i, observation_i)$ pairs. We then form a prompt:
\begin{small}
\begin{verbatim}
You are given a question and the results of several searches.
Based on the search results, answer the question concisely.

Question: {question}

Search 1 (query: "{query_1}"):
{observation_1}

Search 2 (query: "{query_2}"):
{observation_2}

... (more search results) ...

Based on the above search results, provide a concise answer.
Answer:
\end{verbatim}
\end{small}
and have the \emph{reasoning source} model generate up to 50 tokens, stopping at the first newline. The generated text is normalized (lowercase, strip articles, strip punctuation) and compared to the gold answer. Importantly, this protocol uses a different prompt format from the in-loop ReAct evaluation: there is no \texttt{Thought:} cue, no turn-by-turn interleaving, and the agent is not asked to issue further searches. This is why the absolute EM rates in Table~\ref{tab:cross_policy} are much lower than the in-loop $\textsc{Pass}@(1, 5)$ values in Table~\ref{tab:full_pass_kt}: the cross-policy prompt is deliberately stripped to isolate the reasoning-only contribution. The value of the experiment is in the relative ordering across the four conditions, not the absolute numbers.

\begin{table}[h]
\centering
\small
\caption{Cross-policy evaluation on Category C. Denominators differ because each condition restricts to problems for which the strategy source has a successful trajectory to replay.}
\label{tab:cross_policy}
\begin{tabular}{llcc}
\toprule
Strategy Source & Reasoning Source & EM Accuracy & Correct / Attempted \\
\midrule
$\pibase$ & $\pibase$ & $0.130$ & $10 / 77$ \\
$\pirl$   & $\pirl$   & $0.148$ & $12 / 81$ \\
$\pirl$   & $\pibase$ & $0.148$ & $12 / 81$ \\
$\pibase$ & $\pirl$   & $0.143$ & $11 / 77$ \\
\bottomrule
\end{tabular}
\end{table}

\paragraph{Strategy diversity normalization.}
When counting unique query sequences per problem, we lowercase each query, strip surrounding punctuation and whitespace, collapse internal whitespace, and then form the ordered tuple of normalized queries as the sequence key. Without this normalization, cosmetic variations such as trailing periods or capitalization differences would count as distinct strategies and inflate the novelty fraction; with it, the $3\times$ diversity collapse of $\pisft$ vs.\ $\pibase$ is preserved, indicating that the effect is not a tokenization artifact.

\section{Qualitative Analysis of Capability Expansion}
\label{app:qualitative}

To check that the 5:1 per-problem asymmetry of $\pirl$ against $\pibase$ on Category C is not an accident of a single quirky subset of bridge questions, we inspect the five problems in $|\text{only } \pirl|$ and the single problem in $|\text{only } \pibase|$. In all five $\pirl$-only problems, the base agent's 64 samples at $T = 5$ generate a plausible first-hop retrieval (the correct bridge entity appears in the retrieved paragraph) but fail to chain it correctly into a second-hop query: either the agent re-asks the first question rephrased, issues an over-specific query that misses the second gold paragraph, or extracts the wrong bridge entity from the first observation. The RL agent's successful trajectories on the same five problems show a consistent pattern: the first search is nearly identical to the base model's (consistent with the perplexity decomposition's finding that RL queries are in-distribution for $\pibase$), but the reasoning text between Search 1 and Search 2 is more compressed and extracts the bridge entity by a direct copy rather than a paraphrase. The single $\pibase$-only problem is the mirror image: the base agent's first search returns both gold paragraphs in sequence by chance, so no composition is needed; the RL agent samples a slightly longer query that retrieves a distractor on its first attempt. We take the asymmetry as evidence that $\pirl$'s gain on Cat~C is concentrated in the ``how to use the first observation'' step, consistent with the Section~\ref{sec:mechanism} mechanism findings.

\section{Computational Cost}
\label{app:cost}

All experiments run on a shared cluster with 8$\times$ NVIDIA RTX 6000 Ada (48~GB) GPUs. Training takes roughly $55$ hours on a single GPU for $\pirl$ (LoRA rank 16, 8-bit Adam, $\sim 40$M trainable params), dominated by the $2 \times 8 = 16$ rollouts per gradient step with BM25 retrieval in the inner loop. SFT is inexpensive ($\sim 12$ minutes for 3 epochs over 200 expert trajectories). Evaluation is the dominant cost at inference time: for each of the four models, $(100 \text{ problems}) \times (64 \text{ samples}) \times (5 \text{ depths}) = 32000$ trajectories, at approximately $10$--$30$ it/s using vLLM continuous batching on a single GPU, totals $\sim 30$--$60$ GPU-minutes per (model, category) cell. Full pipeline wall-clock for all four models across all three categories (excluding training) is about $8$ GPU-hours.

\section{Full Marginal-Value Data}
\label{app:marginal}

This appendix gives the complete marginal-value data underlying \S\ref{sec:marginal}, Figure~\ref{fig:marginal}, and the budget-allocation recommendation.

\subsection{Fraction-ratio heatmap across the $(k, T)$ grid}
\label{app:marginal_heatmap}

Figure~\ref{fig:marginal_heatmap} shows the ratio $\Delta_T / (\Delta_T + \Delta_k)$ on a $2 \times 3$ grid of heatmaps (rows = Categories B and C, columns = models). Warm cells (ratio $> 0.5$) mean depth is the more valuable direction at that $(k, T)$ operating point; cool cells (ratio $< 0.5$) mean sampling is. Cells where both $\Delta_T$ and $\Delta_k$ fall below $0.005$ are rendered gray (``sat'') to avoid division-by-small-number artifacts. Category A is omitted because $\Delta_T = 0$ throughout (no tool access).

\begin{figure}[h]
\centering
\includegraphics[width=0.95\textwidth]{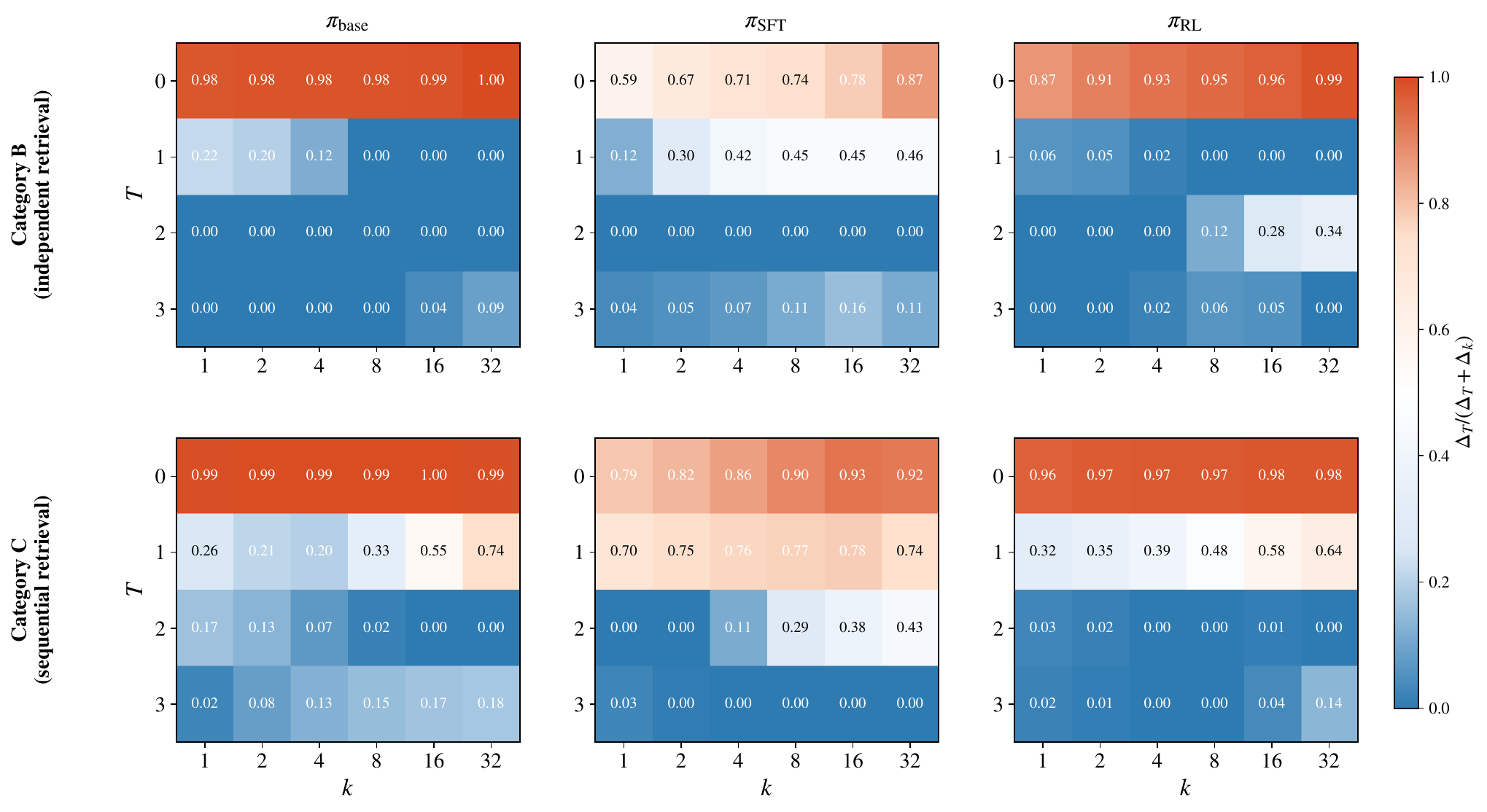}
\caption{Fraction-ratio heatmap $\Delta_T / (\Delta_T + \Delta_k)$ per operating point. Rows are Categories B and C; columns are models ($\pibase$, $\pisft$, $\pirl$). Red cells favor interaction depth; blue cells favor sampling breadth. The warm interior of each Category~C panel (especially the top-left corner, where both $k$ and $T$ are small) is the region where deepening $T$ returns the most, and $\pirl$'s warm region is larger than $\pibase$'s, matching the widening-gap behavior of Figure~\ref{fig:passk}.}
\label{fig:marginal_heatmap}
\end{figure}

\subsection{Budget allocation summary}
\label{app:marginal_budget}


\begin{table}[h]
\centering
\small
\caption{Budget-allocation recommendations derived from $\Delta_T$ and $\Delta_k$. `Recommended $T$' is the smallest depth beyond which $\Delta_T$ drops below $\Delta_k$ at $k = 4$; `cross-over $k$' is the smallest $k$ at which $\Delta_T < \Delta_k$ holds at $T = 1$. Cells marked `---' are saturated or inapplicable.}
\label{tab:budget}
\begin{tabular}{llccl}
\toprule
Category & Model & Rec.\ $T$ & Cross-over $k$ & Budget strategy \\
\midrule
A & $\pi_{\mathrm{base}}$ & 0 & --- & tool unavailable; spend budget on $k$ \\
A & $\pi_{\mathrm{SFT}}$ & 0 & --- & tool unavailable; spend budget on $k$ \\
A & $\pi_{\mathrm{RL}}$ & 0 & --- & tool unavailable; spend budget on $k$ \\
\addlinespace[2pt]
B & $\pi_{\mathrm{base}}$ & 2 & 1 & deepen to $T = 2$, then grow $k$ \\
B & $\pi_{\mathrm{SFT}}$ & 2 & 1 & deepen to $T = 2$, then grow $k$ \\
B & $\pi_{\mathrm{RL}}$ & 2 & 1 & deepen to $T = 2$, then grow $k$ \\
\addlinespace[2pt]
C & $\pi_{\mathrm{base}}$ & 2 & 1 & deepen to $T = 2$, then grow $k$ \\
C & $\pi_{\mathrm{SFT}}$ & 3 & --- & deepen to $T = 3$, then grow $k$ \\
C & $\pi_{\mathrm{RL}}$ & 2 & 1 & deepen to $T = 2$, then grow $k$ \\
\addlinespace[2pt]
\bottomrule
\end{tabular}
\end{table}

\begin{table}[h]
\centering
\small
\caption{Category B (HotPotQA comparison): marginal improvements $\Delta_k$ and per-round $\Delta_T$ at each $(k, T)$ operating point. $\Delta_T$ for $T=3$ is the per-round normalization $(\mathrm{PASS}@(k,5) - \mathrm{PASS}@(k,3))/2$; all other $\Delta_T$ entries are adjacent steps. Dashes indicate cells unavailable for this category.}
\label{tab:delta_b}
\begin{tabular}{cc|cc|cc|cc}
\toprule
 & & \multicolumn{2}{c|}{$\pi_{\mathrm{base}}$} & \multicolumn{2}{c|}{$\pi_{\mathrm{SFT}}$} & \multicolumn{2}{c}{$\pi_{\mathrm{RL}}$} \\
$k$ & $T$ & $\Delta_k$ & $\Delta_T$ & $\Delta_k$ & $\Delta_T$ & $\Delta_k$ & $\Delta_T$ \\
\midrule
1 & 0 & +0.005 & +0.263 & +0.085 & +0.125 & +0.030 & +0.206 \\
2 & 0 & +0.008 & +0.396 & +0.085 & +0.170 & +0.033 & +0.332 \\
4 & 0 & +0.010 & +0.528 & +0.079 & +0.194 & +0.034 & +0.455 \\
8 & 0 & +0.011 & +0.632 & +0.071 & +0.200 & +0.030 & +0.537 \\
16 & 0 & +0.009 & +0.702 & +0.055 & +0.193 & +0.022 & +0.581 \\
32 & 0 & +0.003 & +0.748 & +0.027 & +0.180 & +0.009 & +0.601 \\
\addlinespace[2pt]
1 & 1 & +0.139 & +0.039 & +0.131 & +0.018 & +0.156 & +0.010 \\
2 & 1 & +0.139 & +0.035 & +0.109 & +0.047 & +0.155 & +0.009 \\
4 & 1 & +0.114 & +0.015 & +0.085 & +0.061 & +0.116 & +0.002 \\
8 & 1 & +0.082 & -0.007 & +0.064 & +0.053 & +0.074 & -0.006 \\
16 & 1 & +0.054 & -0.022 & +0.042 & +0.035 & +0.043 & -0.013 \\
32 & 1 & +0.035 & -0.024 & +0.027 & +0.023 & +0.017 & -0.014 \\
\addlinespace[2pt]
1 & 2 & +0.135 & -0.116 & +0.160 & -0.024 & +0.155 & -0.018 \\
2 & 2 & +0.120 & -0.137 & +0.123 & -0.028 & +0.149 & -0.017 \\
4 & 2 & +0.092 & -0.124 & +0.077 & -0.027 & +0.108 & -0.005 \\
8 & 2 & +0.067 & -0.086 & +0.046 & -0.022 & +0.067 & +0.009 \\
16 & 2 & +0.052 & -0.047 & +0.030 & -0.020 & +0.042 & +0.016 \\
32 & 2 & +0.039 & -0.019 & +0.024 & -0.018 & +0.031 & +0.016 \\
\addlinespace[2pt]
1 & 3 & +0.114 & -0.007 & +0.156 & +0.006 & +0.155 & -0.002 \\
2 & 3 & +0.133 & -0.008 & +0.125 & +0.007 & +0.160 & +0.000 \\
4 & 3 & +0.129 & -0.006 & +0.081 & +0.006 & +0.122 & +0.003 \\
8 & 3 & +0.106 & -0.001 & +0.048 & +0.006 & +0.075 & +0.005 \\
16 & 3 & +0.079 & +0.003 & +0.032 & +0.006 & +0.041 & +0.002 \\
32 & 3 & +0.048 & +0.005 & +0.032 & +0.004 & +0.015 & -0.001 \\
\bottomrule
\end{tabular}
\end{table}

\begin{table}[h]
\centering
\small
\caption{Category C (HotPotQA bridge): marginal improvements $\Delta_k$ and per-round $\Delta_T$ at each $(k, T)$ operating point. $\Delta_T$ for $T=3$ is the per-round normalization $(\mathrm{PASS}@(k,5) - \mathrm{PASS}@(k,3))/2$; all other $\Delta_T$ entries are adjacent steps. Dashes indicate cells unavailable for this category.}
\label{tab:delta_c}
\begin{tabular}{cc|cc|cc|cc}
\toprule
 & & \multicolumn{2}{c|}{$\pi_{\mathrm{base}}$} & \multicolumn{2}{c|}{$\pi_{\mathrm{SFT}}$} & \multicolumn{2}{c}{$\pi_{\mathrm{RL}}$} \\
$k$ & $T$ & $\Delta_k$ & $\Delta_T$ & $\Delta_k$ & $\Delta_T$ & $\Delta_k$ & $\Delta_T$ \\
\midrule
1 & 0 & +0.002 & +0.296 & +0.024 & +0.090 & +0.009 & +0.255 \\
2 & 0 & +0.003 & +0.393 & +0.027 & +0.122 & +0.011 & +0.361 \\
4 & 0 & +0.003 & +0.477 & +0.026 & +0.157 & +0.012 & +0.455 \\
8 & 0 & +0.003 & +0.540 & +0.023 & +0.200 & +0.014 & +0.521 \\
16 & 0 & +0.003 & +0.587 & +0.019 & +0.245 & +0.013 & +0.563 \\
32 & 0 & +0.005 & +0.621 & +0.025 & +0.290 & +0.013 & +0.593 \\
\addlinespace[2pt]
1 & 1 & +0.100 & +0.035 & +0.055 & +0.131 & +0.115 & +0.054 \\
2 & 1 & +0.086 & +0.023 & +0.062 & +0.186 & +0.104 & +0.057 \\
4 & 1 & +0.067 & +0.016 & +0.069 & +0.221 & +0.078 & +0.051 \\
8 & 1 & +0.050 & +0.024 & +0.068 & +0.229 & +0.056 & +0.052 \\
16 & 1 & +0.036 & +0.045 & +0.064 & +0.224 & +0.043 & +0.060 \\
32 & 1 & +0.024 & +0.070 & +0.075 & +0.214 & +0.041 & +0.072 \\
\addlinespace[2pt]
1 & 2 & +0.088 & +0.017 & +0.110 & -0.010 & +0.118 & +0.004 \\
2 & 2 & +0.079 & +0.012 & +0.097 & -0.003 & +0.098 & +0.002 \\
4 & 2 & +0.075 & +0.006 & +0.077 & +0.010 & +0.079 & -0.000 \\
8 & 2 & +0.070 & +0.001 & +0.063 & +0.026 & +0.064 & -0.000 \\
16 & 2 & +0.062 & -0.003 & +0.054 & +0.034 & +0.056 & +0.000 \\
32 & 2 & +0.044 & -0.005 & +0.041 & +0.032 & +0.059 & -0.003 \\
\addlinespace[2pt]
1 & 3 & +0.083 & +0.002 & +0.117 & +0.003 & +0.116 & +0.003 \\
2 & 3 & +0.073 & +0.007 & +0.110 & +0.000 & +0.096 & +0.001 \\
4 & 3 & +0.070 & +0.010 & +0.093 & -0.003 & +0.079 & -0.001 \\
8 & 3 & +0.066 & +0.012 & +0.071 & -0.006 & +0.065 & -0.001 \\
16 & 3 & +0.059 & +0.012 & +0.052 & -0.009 & +0.052 & +0.002 \\
32 & 3 & +0.050 & +0.011 & +0.040 & -0.011 & +0.042 & +0.007 \\
\bottomrule
\end{tabular}
\end{table}

\begin{table}[h]
\centering
\small
\caption{Sampling-sufficiency check. For each (model, category) we compare $\mathrm{PASS}@(k,T)$ computed on a random sub-sample of $n' \in \{32, 48\}$ trajectories against the full $n = 64$ estimate, averaged across $k \in \{1,4,16,64\}$ (restricted to $k \leq n'$) and all available $T$. Deviations are small relative to the binomial standard error $\sqrt{p(1{-}p)/n} \approx 0.06$ at $p \approx 0.5$, indicating $n = 64$ is sufficient for the comparisons reported in the main text.}
\label{tab:sampling_sufficiency}
\begin{tabular}{llccc}
\toprule
Category & Model & $n'$ & Mean $|\Delta|$ & Max $|\Delta|$ \\
\midrule
A & $\pi_{\mathrm{base}}$ & 32 & 0.0032 & 0.0049 \\
A & $\pi_{\mathrm{base}}$ & 48 & 0.0019 & 0.0047 \\
A & $\pi_{\mathrm{SFT}}$ & 32 & 0.0046 & 0.0075 \\
A & $\pi_{\mathrm{SFT}}$ & 48 & 0.0006 & 0.0011 \\
A & $\pi_{\mathrm{RL}}$ & 32 & 0.0102 & 0.0202 \\
A & $\pi_{\mathrm{RL}}$ & 48 & 0.0014 & 0.0023 \\
B & $\pi_{\mathrm{base}}$ & 32 & 0.0046 & 0.0194 \\
B & $\pi_{\mathrm{base}}$ & 48 & 0.0027 & 0.0067 \\
B & $\pi_{\mathrm{SFT}}$ & 32 & 0.0049 & 0.0147 \\
B & $\pi_{\mathrm{SFT}}$ & 48 & 0.0036 & 0.0100 \\
B & $\pi_{\mathrm{RL}}$ & 32 & 0.0057 & 0.0182 \\
B & $\pi_{\mathrm{RL}}$ & 48 & 0.0028 & 0.0089 \\
C & $\pi_{\mathrm{base}}$ & 32 & 0.0029 & 0.0114 \\
C & $\pi_{\mathrm{base}}$ & 48 & 0.0018 & 0.0043 \\
C & $\pi_{\mathrm{SFT}}$ & 32 & 0.0050 & 0.0138 \\
C & $\pi_{\mathrm{SFT}}$ & 48 & 0.0016 & 0.0050 \\
C & $\pi_{\mathrm{RL}}$ & 32 & 0.0048 & 0.0119 \\
C & $\pi_{\mathrm{RL}}$ & 48 & 0.0024 & 0.0063 \\
\bottomrule
\end{tabular}
\end{table}

\begin{table}[h]
\centering
\small
\caption{Bootstrap 95\% confidence intervals for the capability-boundary statistics reported in Table~\ref{tab:capability}, computed over 1000 Bernoulli bootstrap replicates per problem using the empirical $c_T/n$ as each problem's success rate. The RL vs.\ base and SFT vs.\ base asymmetries on Category C remain far from zero under the bootstrap, so the $5$ vs.\ $1$ and $3$ vs.\ $7$ splits are not artifacts of sampling noise.}
\label{tab:bootstrap_ci}
\begin{tabular}{llccc}
\toprule
Category & Quantity & Mean & 95\% CI & Point est.\ (main text) \\
\midrule
A & $|\mathcal{B}^{(C)}_{\mathrm{RL}}|$ & 82.0 & [79.0, 84.0] & 84 \\
A & $|\mathcal{B}_{\mathrm{RL}} \setminus \mathcal{B}_{\mathrm{base}}|$ & 3.1 & [1.0, 6.0] & 3 \\
A & $|\mathcal{B}_{\mathrm{base}} \setminus \mathcal{B}_{\mathrm{RL}}|$ & 3.4 & [1.0, 6.0] & 3 \\
A & $|\mathcal{B}_{\mathrm{SFT}} \setminus \mathcal{B}_{\mathrm{base}}|$ & 4.8 & [2.0, 8.0] & 5 \\
A & $|\mathcal{B}_{\mathrm{base}} \setminus \mathcal{B}_{\mathrm{SFT}}|$ & 2.2 & [0.0, 5.0] & 0 \\
\addlinespace[2pt]
B & $|\mathcal{B}^{(C)}_{\mathrm{RL}}|$ & 84.3 & [82.0, 86.0] & 86 \\
B & $|\mathcal{B}_{\mathrm{RL}} \setminus \mathcal{B}_{\mathrm{base}}|$ & 6.1 & [3.0, 10.0] & 5 \\
B & $|\mathcal{B}_{\mathrm{base}} \setminus \mathcal{B}_{\mathrm{RL}}|$ & 1.0 & [0.0, 2.0] & 1 \\
B & $|\mathcal{B}_{\mathrm{SFT}} \setminus \mathcal{B}_{\mathrm{base}}|$ & 7.4 & [5.0, 11.0] & 5 \\
B & $|\mathcal{B}_{\mathrm{base}} \setminus \mathcal{B}_{\mathrm{SFT}}|$ & 2.3 & [1.0, 4.0] & 2 \\
\addlinespace[2pt]
C & $|\mathcal{B}^{(C)}_{\mathrm{RL}}|$ & 77.2 & [74.0, 80.0] & 81 \\
C & $|\mathcal{B}_{\mathrm{RL}} \setminus \mathcal{B}_{\mathrm{base}}|$ & 5.0 & [2.0, 8.0] & 5 \\
C & $|\mathcal{B}_{\mathrm{base}} \setminus \mathcal{B}_{\mathrm{RL}}|$ & 2.7 & [1.0, 5.0] & 1 \\
C & $|\mathcal{B}_{\mathrm{SFT}} \setminus \mathcal{B}_{\mathrm{base}}|$ & 4.2 & [2.0, 7.0] & 3 \\
C & $|\mathcal{B}_{\mathrm{base}} \setminus \mathcal{B}_{\mathrm{SFT}}|$ & 8.7 & [7.0, 11.0] & 7 \\
\addlinespace[2pt]
\bottomrule
\end{tabular}
\end{table}

\section{Sampling Sufficiency}
\label{app:sampling}

To verify that $n = 64$ trajectories per $(q, T)$ cell is sufficient for the comparisons we report, we randomly sub-sample the 64 rollouts to $n' \in \{32, 48\}$ with a fixed seed and recompute category-averaged $\textsc{Pass}@(k, T)$. Table~\ref{tab:sampling_sufficiency} reports the mean and maximum absolute deviation between the sub-sampled average and the full-$n$ average, taken over all $(k, T)$ cells in the evaluation grid. Mean deviations are under $0.006$ and maxima under $0.02$, well below the binomial standard error $\sqrt{p(1{-}p)/n} \approx 0.06$ at $p \approx 0.5$, and an order of magnitude smaller than the cross-model gaps we rely on in \S\ref{sec:results}. We conclude that $n = 64$ is more than sufficient for the conclusions of this paper.

\section{Per-problem Analysis for Categories A and B}
\label{app:per_problem_ab}

For completeness we repeat Figure~\ref{fig:capability}'s per-problem format for Categories A and B, which do not feature in the main text because their capability-set asymmetries are small and symmetric between models. Figures \ref{fig:per_problem_a} and \ref{fig:per_problem_b} show per-problem $\textsc{Pass}@(64, T_{\max})$ (left) and capability-set decomposition vs.\ $\pibase$ (right). On Category A, the three models produce nearly identical columns in the sorted heatmap, and the only-model and only-base slices are balanced, consistent with the negative-control status of Category A. On Category B, $\pisft$, $\piexplore$, and $\pirl$ each add a handful of problems over $\pibase$ and lose a handful, producing the symmetric pattern reported in Table~\ref{tab:capability} (rows B).

\begin{figure}[h]
\centering
\includegraphics[width=\textwidth]{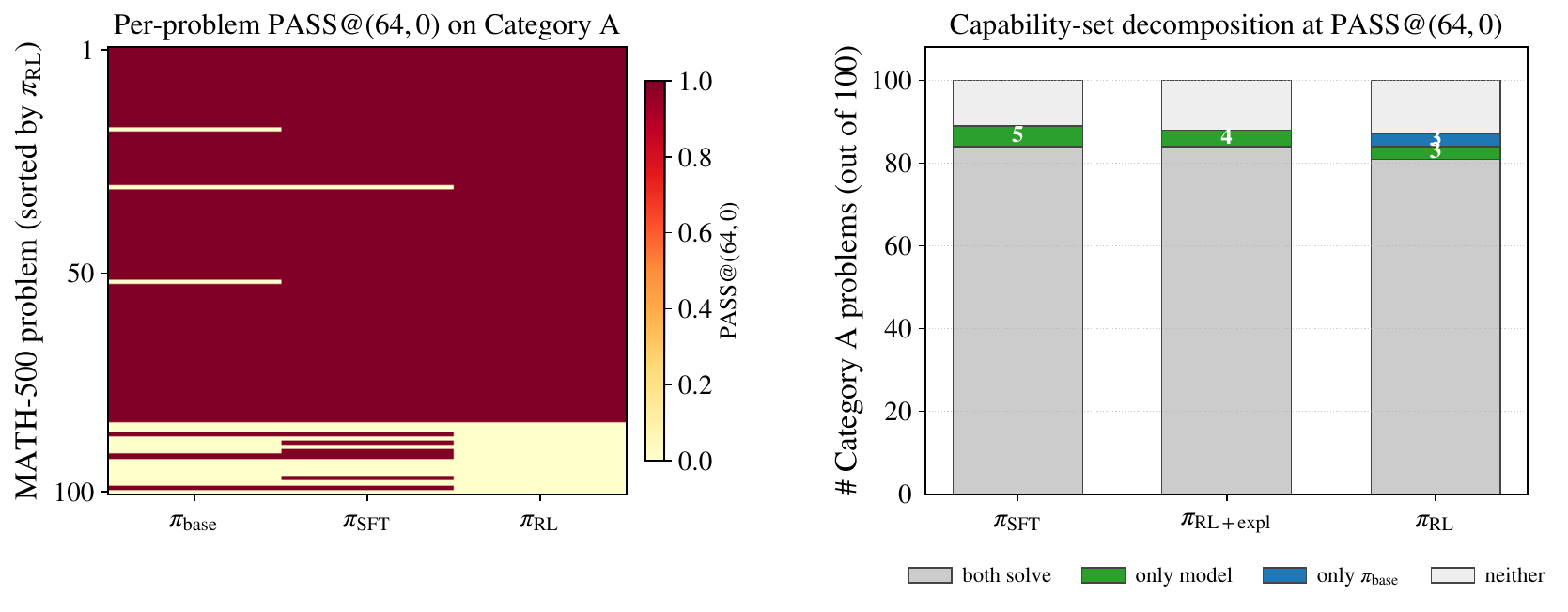}
\caption{Per-problem $\textsc{Pass}@(64, 0)$ on Category A (MATH-500). The three models' columns are nearly identical; the capability-set decomposition is symmetric with only-model and only-base counts both near three.}
\label{fig:per_problem_a}
\end{figure}

\begin{figure}[h]
\centering
\includegraphics[width=\textwidth]{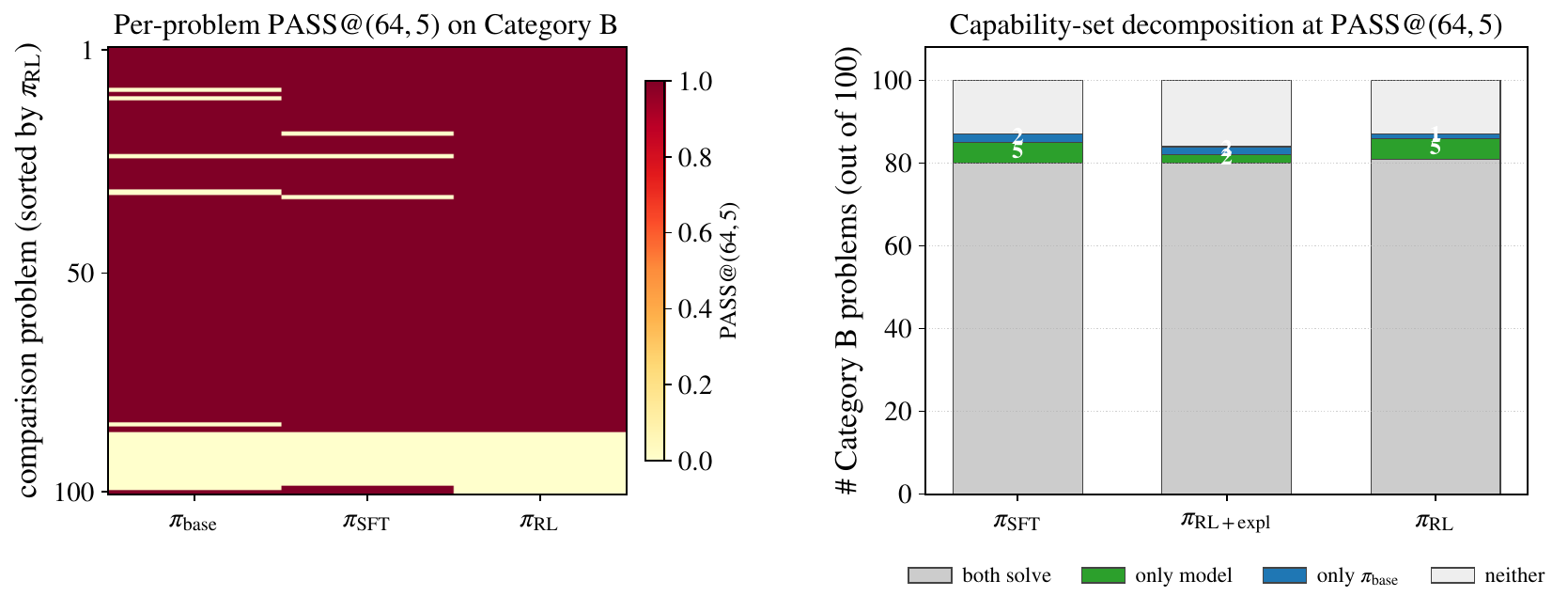}
\caption{Per-problem $\textsc{Pass}@(64, 5)$ on Category B (comparison questions). $\pisft$ and $\pirl$ produce similar capability-set decompositions, each adding $\sim 5$ new bridge problems over $\pibase$ and losing $\sim 2$.}
\label{fig:per_problem_b}
\end{figure}

\section{Bootstrap Confidence Intervals}
\label{app:bootstrap}

We assess the sampling uncertainty in the Table~\ref{tab:capability} capability-boundary statistics via a Bernoulli bootstrap. For each test problem $q$ and model $\pi$ we treat $c_T(q, \pi)/n$ as $\pi$'s success rate on $q$ at depth $T$; we then generate 1000 bootstrap replicates by redrawing $n = 64$ Bernoulli trials per problem with the observed rate, and recompute the capability boundary $\mathcal{B}_T$, the set-difference cardinalities, and the three-way decomposition in each replicate. Table~\ref{tab:bootstrap_ci} reports the mean and 95\% percentile CIs. On Category C, $|\mathcal{B}_{\text{RL}} \setminus \mathcal{B}_{\text{base}}|$ has 95\% CI well above the analogous $\pisft$ statistic, and $|\mathcal{B}_{\text{base}} \setminus \mathcal{B}_{\text{SFT}}|$'s lower bound is 7, confirming that the matched-data SFT regression and the RL expansion are both significantly different from zero under the bootstrap.

\section{Category C per-$T$ Pass Curves}
\label{app:per_t}

Figure~\ref{fig:per_t} is an alternative visualization of the Category C row of Figure~\ref{fig:heatmaps}: instead of a heatmap, we plot $\textsc{Pass}@(k, T)$ vs.\ $k$ with one curve per $T$ value, and a separate subplot for each model. This view makes the effect of deepening interaction directly readable: the spacing between adjacent-$T$ curves is the per-round $\Delta_T$.

\begin{figure}[h]
\centering
\includegraphics[width=\textwidth]{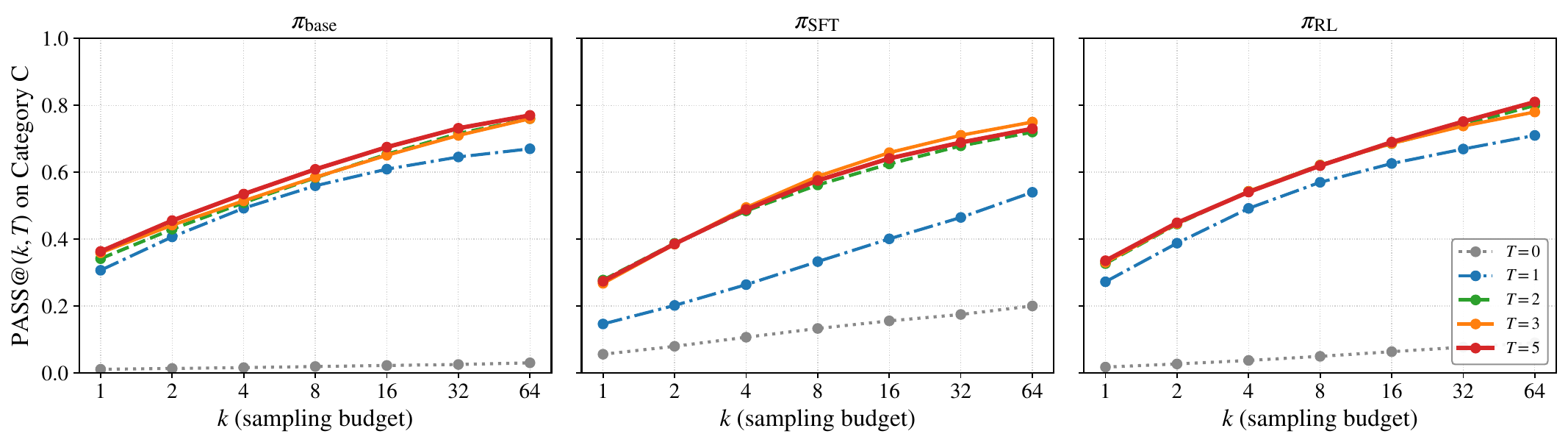}
\caption{Category C pass-curves by interaction depth $T$, one subplot per model. The curves for $T = 2, 3, 5$ cluster tightly together for all three models (the saturation observed in \S\ref{sec:perproblem}), with the largest $T$-direction jump between $T = 0$ and $T = 1$. $\pirl$'s cluster is shifted upward relative to $\pibase$'s across the entire $k$ axis.}
\label{fig:per_t}
\end{figure}

\section{SFT Perplexity Decomposition}
\label{app:sft_ppl}

Figure~\ref{fig:sft_ppl} complements Figure~\ref{fig:mechanism}(a) by computing the same base-model perplexity decomposition on 200 successful $\pisft$ Category C trajectories, then plotting both $\pisft$ and $\pirl$ distributions side-by-side. We find $\text{PPL}_{\text{search}}^{\text{SFT}} = 2.73$ (median $2.38$) and $\text{PPL}_{\text{reason}}^{\text{SFT}} = 3.93$ (median $3.56$), both materially higher than the corresponding $\pirl$ values of $2.07$ and $3.08$. In other words, the base model finds \emph{all} of SFT's trajectory components more surprising than RL's trajectory components: SFT has replaced the base distribution on both the retrieval and the reasoning side, while RL has stayed closer to base on both sides. This matches the strategy-diversity evidence in \S\ref{sec:mechanism}: SFT's mean novelty fraction is $97.7\%$ of sequences (almost complete distribution replacement), while RL's is $83.9\%$.

\begin{figure}[h]
\centering
\includegraphics[width=0.85\textwidth]{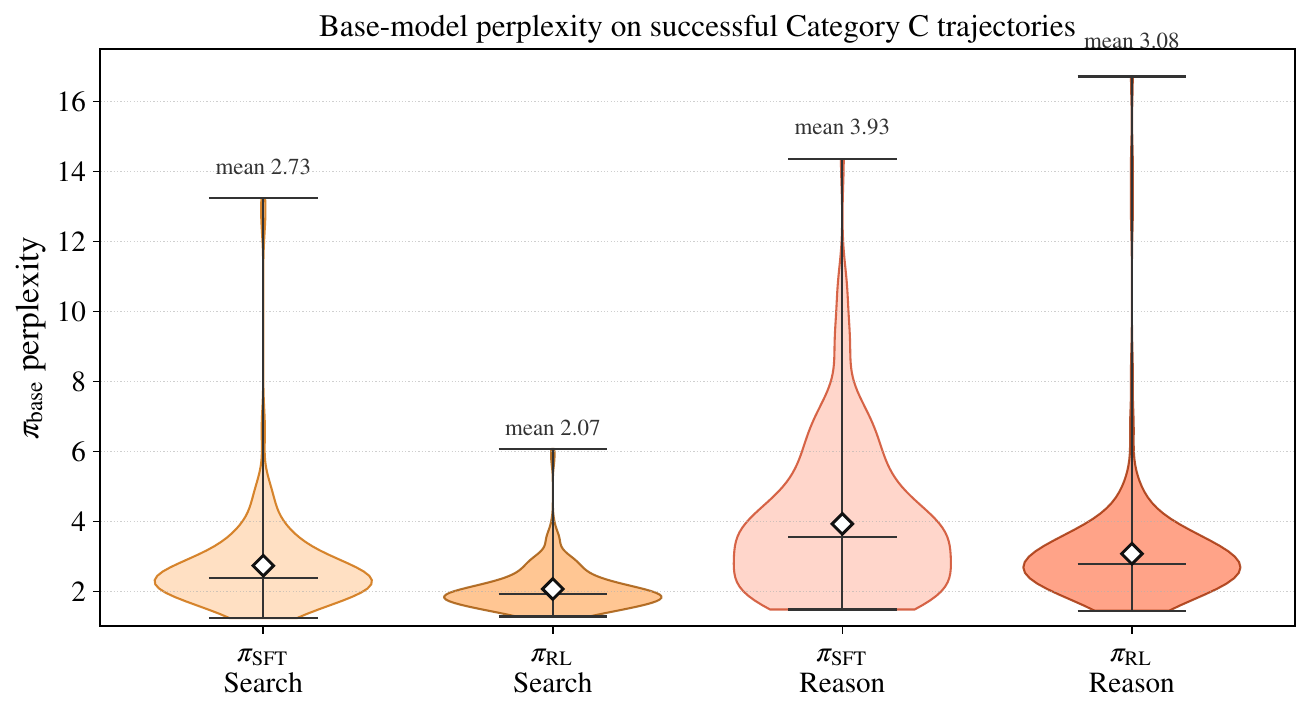}
\caption{$\pibase$'s perplexity on 200 successful Category C trajectories for $\pisft$ and $\pirl$, split into search-query tokens and reasoning tokens. Both $\text{PPL}_{\text{search}}$ and $\text{PPL}_{\text{reason}}$ are higher for $\pisft$ than for $\pirl$; SFT has displaced the base distribution on both sides, while RL has re-weighted within it.}
\label{fig:sft_ppl}
\end{figure}
\newpage

\end{document}